\documentclass{article}
\usepackage{algorithmic}

 \usepackage[preprint]{neurips_2026}

\usepackage[utf8]{inputenc} 
\usepackage[T1]{fontenc}    

\usepackage{hyperref}       
\usepackage{url}            
\usepackage{booktabs}       
\usepackage{amsfonts}       
\usepackage{xfrac}          
\providecommand{\nicefrac}[2]{\sfrac{#1}{#2}}
\usepackage{microtype}      
\usepackage{xcolor}         
\usepackage{graphicx}
\providecommand{\cdashline}[1]{\cline{#1}}
\IfFileExists{multirow.sty}{\usepackage{multirow}}{\providecommand{\multirow}[3]{##3}}
\usepackage{multicol}
\IfFileExists{makecell.sty}{%
  \usepackage{makecell}%
}{%
  \providecommand{\makecell}[2][]{\begin{tabular}{@{}c@{}}##2\end{tabular}}%
}
\usepackage{array}

\usepackage{booktabs} 

\usepackage{hyperref}
\IfFileExists{adjustbox.sty}{%
  \usepackage{adjustbox}%
}{%
  \newenvironment{adjustbox}[1]{}{}%
}

\usepackage{amsmath}
\usepackage{amssymb}
\usepackage{mathtools}
\usepackage{amsthm}
\IfFileExists{cancel.sty}{%
  \usepackage{cancel}%
}{%
  \providecommand{\bcancel}[1]{##1}%
}
\usepackage{tikz}
\usepackage{algorithm}
\usepackage[font=small]{caption}
\IfFileExists{wrapfig.sty}{%
  \usepackage{wrapfig}%
}{%
  \newenvironment{wrapfigure}[2]{\begin{figure}[t]}{\end{figure}}%
}
\usepackage{caption}
\IfFileExists{cleveref.sty}{%
  \usepackage[capitalize,noabbrev]{cleveref}%
}{%
  \providecommand{\Cref}[1]{\ref{##1}}%
  \providecommand{\cref}[1]{\ref{##1}}%
}

\theoremstyle{plain}
\newtheorem{theorem}{Theorem}[section]
\newtheorem{proposition}[theorem]{Proposition}
\newtheorem{lemma}[theorem]{Lemma}

\theoremstyle{definition}

\newtheorem{assumption}[theorem]{Assumption}
\theoremstyle{remark}
\newtheorem{remark}[theorem]{Remark}

\newenvironment{talign*}
{\csname align*\endcsname}
{\endalign}

\title{Outlier-robust Diffusion Posterior Sampling for Bayesian Inverse Problems}

\author{%
  Yiming Yang$^{1}$\thanks{First Author: Email Yiming Yang to \url{zcahyy1@ucl.ac.uk}.}  \quad
  Xiaoyuan Cheng$^{1}$ \quad
  Yi He$^{1}$ \quad
  Kaiyu Li$^{1}$ \quad
  Wenxuan Yuan$^{2}$ \quad
  Zhuo Sun$^{3,2}$\thanks{Corresponding Author. Correspondence to Zhuo Sun: \url{sunzhuo@mail.shufe.edu.cn}.}\\
  $^{1}$University College London, \quad
  $^{2}$Imperial College London,\\
  $^{3}$Shanghai University of Finance and Economics 
}

\begin{document}

\maketitle

\begin{abstract}
  Diffusion models have emerged as powerful learned priors for Bayesian inverse problems (BIPs). Diffusion-based solvers rely on a presumed likelihood for the observations in BIPs to guide the generation process. Likelihood misspecification is common in practical BIPs and is known to degrade recovery performance, particularly under outlier contamination. We investigate this problem by first characterizing the induced posterior deviation and proving the \emph{stability} of diffusion-based solvers for linear BIPs. Our stability analysis further reveals potential robustness deficiencies of existing diffusion-based solvers under outlier-contaminated measurements. To address this issue, we propose a simple yet effective solution: \emph{robust diffusion posterior sampling}, which is provably \emph{outlier-robust} for linear BIPs and compatible with existing gradient-based posterior samplers. Empirical results from scientific inverse problems and natural image tasks demonstrate the effectiveness and robustness of our method, with consistent performance gains in challenging scenarios involving outlier contamination for both linear and nonlinear tasks.
\end{abstract}

\section{Introduction}
\label{sec: introduction}
The goal of Inverse Problems (IPs) is to recover the unknown variable $\boldsymbol{x}$ given noisy measurements $\boldsymbol{y}=F(\boldsymbol{x}) + \boldsymbol{\epsilon}$. Because measurements are limited and noisy, IPs are often ill-posed \citep{de2005learning}. IPs appear in various domains, such as medical imaging \citep{song2021solving}, geophysics \citep{sambridge2002monte}, and beyond. Bayesian inference mitigates the ill-posedness via a prior $p(\boldsymbol{x})$ to regularize the solution based on prior knowledge. The likelihood $p(\boldsymbol{y}|\boldsymbol{x})$ quantifies how well $F(\boldsymbol{x})$ explains the measurements \(\boldsymbol{y}\). Applying Bayes’ rule yields the posterior distribution $p(\boldsymbol{x}|\boldsymbol{y})$. This formulation is referred to as a Bayesian inverse problem (BIP), and the posterior is viewed as the solution \citep{sullivan2015introduction}.

Despite the simple formulation of BIPs, determining a suitable prior $p(\boldsymbol{x})$ is both critical and challenging, particularly for high-dimensional complex data distributions. Apart from prior choices in existing works \citep{knapik2011bayesian, wang2017bayesian, hosseini2017well,  chung2023diffusion, cardoso2024monte}, probabilistic diffusion models (DMs) \citep{sohl2015deep,song2019generative,ho2020denoising,song2021score} have emerged as powerful data-driven priors for solving inverse problems across various scenarios and applications \citep{zheng2025inversebench}. DMs identify the target data distribution by reversing a forward diffusion process via a learned time-dependent score function $\nabla_{\boldsymbol{x}_t}\log{p(\boldsymbol{x}_t)}$, which progressively transports samples from Gaussian noise at time $T$ to data samples from the target distribution $p(\boldsymbol{x}_0)$. 
Unlike standard BIPs, applying diffusion priors requires estimating intermediate posteriors $p(\boldsymbol{x}_t|\boldsymbol{y})$ along the reverse diffusion process. These intermediate distributions are generally intractable. A widely adopted approach to address the intractability relies on \textit{Tweedie’s formula} \citep{miyasawa1961empirical} to obtain an estimate of the denoised conditional $p(\boldsymbol{x}_0|\boldsymbol{x}_t)$, effectively providing a point estimate (or a local Gaussian approximation) \citep{song2021denoising,meng2021estimating,chung2023diffusion,song2023pseudoinverse,boys2024tweedie}. This allows for a tractable likelihood approximation that combines with the diffusion prior to estimate intermediate posteriors and progressively solves the IP; see \citep{daras2024survey} for a review.
\begin{figure*}
    \centering    
    \includegraphics[width=\linewidth]{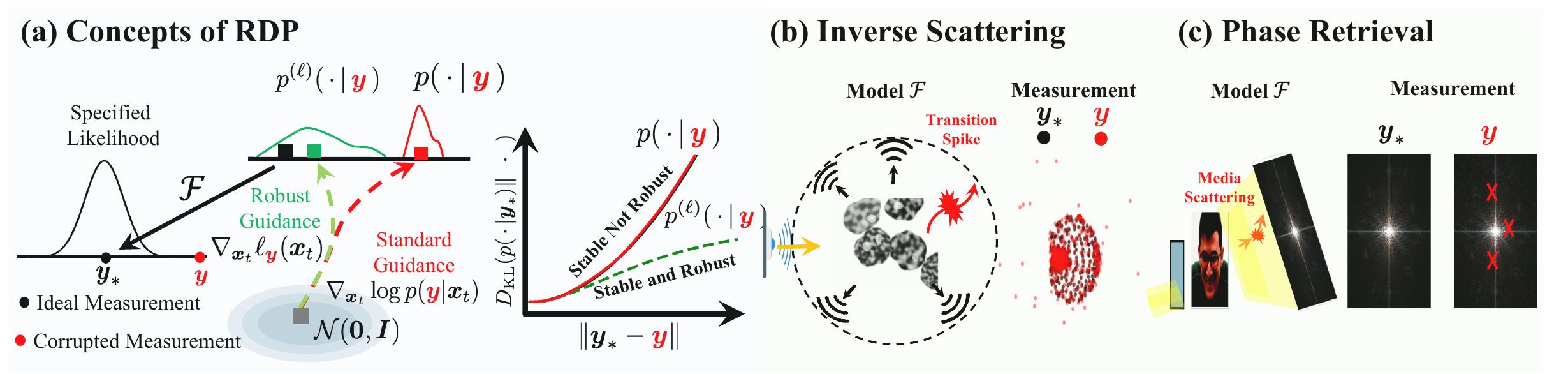}
    \caption{\small Overview. (a) Our proposed method achieves robust recovery under corruption. Visualization of the measurement model and data corruption, showing the sources of measurement corruption in (b) inverse scattering and (c) phase retrieval.}
    \label{fig: demo}
\end{figure*}

\paragraph{On Stability for Diffusion-Based Solvers.} 
Diffusion-based solvers have shown strong empirical performance for BIPs, but their stability with respect to measurement error remains less explicitly characterized. Prior works analyze surrogate likelihoods based on Tweedie approximations \citep{chung2023diffusion,song2023loss}, while others study stability under forward-model shifts in linear BIPs \citep{renaud2024plugandplay} or sampling stability for inexact Langevin-type algorithms \citep{renaud2026from}. Complementary to these results, we study stability under measurement perturbations. Specifically, we show that the posterior induced by DPS \citep{chung2023diffusion} is stable under measurement error and quantify the posterior deviation, which scales with the magnitude of measurement noise.

\paragraph{On Robustness for Diffusion-Based Solvers under Outlier Contamination.} A finding from our stability results indicates the potential lack of robustness of existing diffusion-based methods for BIPs under outlier contamination. Prior work has studied sensitivity to denoiser shifts and forward-model shifts in linear BIPs~\citep{renaud2024plugandplay}. 
However, robustness to outlier contamination in the measurement noise remains largely underexplored.  In BIPs, the likelihood $p(\boldsymbol{y}|\boldsymbol{x})$ is determined by the measurement model \(F\) and an assumed noise distribution \(\boldsymbol{\epsilon}\) which can be misspecified in practice. For example, in optical tasks such as inverse scattering and phase retrieval, speckle noise or unexpected scattering from inhomogeneous media can produce heavy-tailed noise or outlier contamination, thereby violating the common Gaussian assumptions (see Figure~\ref{fig: demo}(b)–(c)). In this setting, stable but non-robust diffusion-based solvers can drive the posterior far from the target one. One way to deal with likelihood misspecifications is to temper the likelihood with heuristic \emph{temperatures} \citep{zhu2023denoising,chung2023diffusion,song2023pseudoinverse,song2023loss,wu2023practical,cardoso2024monte}, thereby globally down-weighting the likelihood term in the reverse diffusion processes. However, this global control is intended to stabilize the reverse sampling dynamics rather than explicitly mitigate the effect of likelihood misspecifications for BIPs.

\paragraph{Contributions.}
To address this problem, we provide a systematic investigation through theoretical and empirical analysis. We summarize our contributions as follows: 
\begin{itemize}
    \item We establish observations and theoretical results on the \emph{stability of diffusion posterior sampling} for linear Bayesian inverse problems under measurement noise, providing insight into their empirical effectiveness.
    
    \item We show that this stability analysis also reveals a vulnerability under outlier contamination. Motivated by this finding, we introduce \emph{robust diffusion posterior sampling} (RDP), a GBI-inspired framework for robust BIPs. RDP serves as a plug-and-play robust guidance module for existing gradient-based diffusion samplers, as illustrated in Figure~\ref{fig: demo}(a), and provides theoretical guaranties of \emph{outlier robustness}.
    
    \item We empirically validate our theory across linear and nonlinear scientific inverse problems and image restoration tasks. The results show that RDP mitigates the outlier contamination visualized in Figure~\ref{fig: demo}(b,c), improves performance under likelihood misspecification, and maintains competitive performance in well-specified settings.
\end{itemize}


\section{Background}
\label{sec: background}
In this section, we provide the background on BIPs and diffusion-based solvers for these problems.
 
\subsection{Bayesian Inverse Problems}
\label{subsec: BIP}
An inverse problem aims to recover the unknown variables
$\boldsymbol{x}\in\mathcal{X}\subset\mathbb{R}^{d_x}$ with compact \(\mathcal{X}\) from noisy observations $\boldsymbol{y}\in\mathbb{R}^{d_y}$ given by:
\begin{equation}
    \label{eq: general_inverse_problem}
    \boldsymbol{y}=F(\boldsymbol{x}) + \boldsymbol{\epsilon},
\end{equation}
where $F\colon\mathcal{X}\rightarrow\mathbb{R}^{d_y}$ denotes the measurement model, and $\boldsymbol{\epsilon}\in\mathbb{R}^{d_y}$ denotes the additive measurement noise \citep{sullivan2015introduction}. We consider linear inverse problems, i.e. $F \in \mathbb{R}^{d_y\times d_x}$ for theoretical analysis in this work; though we also observe that our method shows promising empirical results for nonlinear BIPs as well (see \Cref{sec: experiments}). Typically $d_x \gg d_y$, the inverse problem is under-determined and requires regularization from prior knowledge.

Accordingly, the BIP models the unknown \(\boldsymbol{x}\) as a random variable by specifying
a prior \(p(\boldsymbol{x})\) on \(\mathcal{X}\) and a likelihood \(p(\boldsymbol{y}|\boldsymbol{x})\) induced by 
\eqref{eq: general_inverse_problem} \citep{stuart2010inverse,sullivan2015introduction}. In this paper, we focus on the commonly used Gaussian likelihood, which factorizes across measurement components, i.e., \(p(\boldsymbol{y}|\boldsymbol{x})=\prod_{i=1}^{d_y}p(y_i|\boldsymbol{x})\). Given an observation $\boldsymbol{y}$, Bayes’ rule yields the posterior 
\(p(\boldsymbol{x}|\boldsymbol{y})\propto p(\boldsymbol{y}|\boldsymbol{x})\,p(\boldsymbol{x}).
\) However, traditional priors, such as the standard Gaussian, are often inadequate for representing the complex distributions of real-world data. Diffusion models, instead, offer a compelling choice for BIPs by learning expressive, data-driven priors \citep{song2021solving, song2021score}.

\subsection{Diffusion Models for Inverse Problems}
\label{subsec: Diffusion Models for Inverse Problems}
Diffusion models provide a powerful framework for learning complex, high-dimensional distributions. In this work, we adopt the continuous-time Variance Preserving (VP) formulation. The forward process transforms a data sample $\boldsymbol{x}_0$ into Gaussian noise over $t\in[0,T]$, governed by the stochastic differential equation (SDE)
$\mathrm{d}\boldsymbol{x}_t = -\frac{1}{2}\beta(t)\boldsymbol{x}_t\,\mathrm{d}t+\sqrt{\beta(t)} \mathrm{d}\boldsymbol{w}$,
where $\beta(t)$ is a predefined noise schedule and $\boldsymbol{w}_t$ is a standard Wiener process. The key insight is that the forward process can be reversed in time using: 
\begin{align}
 \label{eq: reverse SDE}   
 \mathrm{d}\boldsymbol{x}_t = -\beta(t)[\frac{\boldsymbol{x}_t}{2}+\nabla_{\boldsymbol{x}_t}\log{p(\boldsymbol{x}_t)}]\mathrm{d} t+\sqrt{\beta(t)}\mathrm{d}\boldsymbol{w}_t.
\end{align}
The critical component required to solve the reverse SDE is the score function, $\nabla_{\boldsymbol{x}_t}\log{p_t(\boldsymbol{x}_t)}$, which is approximated by a time-dependent neural network $\boldsymbol{s}_{\boldsymbol{\theta}}(\boldsymbol{x}_t,t)$. This network is trained via denoising score matching, minimizing the expected squared error between the predicted and true conditional scores. Once trained, the score network $\boldsymbol{s}_{\boldsymbol{\theta}}$ is used to generate samples by solving the reverse SDE, starting from $\boldsymbol{x}_T\sim\mathcal{N}(\boldsymbol{0},\boldsymbol{I})$.

\vspace{-0.5em}
\paragraph{Posterior Sampling with Diffusion Models.} Solving a BIP requires sampling from the posterior $p(\boldsymbol{x}_0|\boldsymbol{y})$. This can be achieved by guiding the reverse SDE using the score of the posterior distribution at each time step $\nabla_{\boldsymbol{x}_t}\log{p_t(\boldsymbol{x}_t|\boldsymbol{y})}$. Using Bayes' theorem, this posterior score can be decomposed into a prior and a likelihood:
\begin{equation*}
    \label{eq: guided reverse SDE}
    \nabla_{\boldsymbol{x}_t}\log p_t(\boldsymbol{x}_t|\boldsymbol{y}) = \underbrace{\nabla_{\boldsymbol{x}_t}\log p_t(\boldsymbol{x}_t)}_{\text{Prior Score}} + \underbrace{\nabla_{\boldsymbol{x}_t}\log p(\boldsymbol{y}|\boldsymbol{x}_t)}_{\text{Likelihood Score}}.
\end{equation*}
The prior score is readily available from the pre-trained score network, i.e., $\boldsymbol{s}_{\boldsymbol{\theta}}(\boldsymbol{x}_t,t)$. The likelihood term is difficult to derive, as the time-dependent likelihood function $p(\boldsymbol{y}|\boldsymbol{x_t})=\int p(\boldsymbol{y}|\boldsymbol{x}_0)p(\boldsymbol{x}_0|\boldsymbol{x}_t)\,\mathrm{d}\boldsymbol{x}_0=\mathbb{E}\left[p(\boldsymbol{y}|\boldsymbol{x}_0)\right]$ is not tractable in general.

An efficient strategy is to approximate the intractable likelihood $p(\boldsymbol{y}|\boldsymbol{x}_t)$ with a surrogate $p(\boldsymbol{y}|\boldsymbol{x}_t)\approx \tilde{p}(\boldsymbol{y}|\hat{\boldsymbol{x}}_0(\boldsymbol{x}_t))$, which is the likelihood evaluated at \(\hat{\boldsymbol{x}}_0(\boldsymbol{x}_t)\), the estimate of \(\boldsymbol{x}_0\) at step $t$ from the diffusion model \citep{stein1981estimation, efron2011tweedie,graikos2022diffusion, chung2023diffusion}. We refer to this surrogate likelihood as \textit{DPS} likelihood. While the DPS approximation has proven effective for well-specified likelihood models \citep{song2021denoising,ho2022video, wu2023practical,song2023loss}, its robustness to likelihood misspecifications remains a concern. Likelihood misspecifications can arise from the misspecified noise distribution or from the presence of a contaminated subset of observations generated under a different data-generating mechanism. In this work, we investigate the impact of such misspecifications and focus on the following questions:
\begin{center}
    \textit{Is diffusion posterior sampling robust to outlier contamination for linear BIPs? If not, can we propose a solution that guaranties robust recovery without losing effectiveness?}
\end{center}

\section{Method}
\label{sec: method}
This section begins by establishing the stable properties of the standard diffusion sampler, which also shows its vulnerability to the likelihood misspecification discussed in Section \ref{subsec:stability}. We then formulate a robust strategy with theoretical guaranties and detail its practical implementation for general BIPs with diffusion solvers in Section \ref{subsec:RDP}. We discuss how our robust approach performs in \emph{well-specification} settings.  
We consider observations contaminated by outliers: \(
\boldsymbol{y} = F\boldsymbol{x} + \boldsymbol{\epsilon},
\)
where \(\boldsymbol\epsilon\) may contain both nominal measurement noise and sparse, large-magnitude outlier corruption. Hereby, the method achieves the design of a posterior that is robust to the presence of outliers.

\subsection{Stability and Robustness of Diffusion Posterior}
\label{subsec:stability}
Following the likelihood approximation in Section \ref{subsec: Diffusion Models for Inverse Problems}, the likelihood score is evaluated as 
\(\nabla_{\boldsymbol{x}_t}\log{p(\boldsymbol{y}|\boldsymbol{x}_t)}\approx \nabla_{\boldsymbol{x}_t}\log{\tilde{p}}(\boldsymbol{y}|\hat{\boldsymbol{x}}_0(\boldsymbol{x}_t)).
\)
The \(\hat{\boldsymbol{x}}_0(\boldsymbol{x}_t)\) is the denoised estimate of the clean data given $\boldsymbol{x}_t$, defined by: \[\hat{\boldsymbol{x}}_0(\boldsymbol{x}_t)=\mathbb{E}(\boldsymbol{x}_0|\boldsymbol{x}_t)\approx\frac{1}{\sqrt{\alpha(t)}}(\boldsymbol{x}_t+(1-\alpha(t))\boldsymbol{s}_{\boldsymbol{\theta}}(\boldsymbol{x}_t,t)),
\]
where \(\alpha(t)=\exp{(-\int_{0}^t\beta(\tau)\mathrm{d}\tau})\).
Then, one can perform posterior sampling by solving the reverse SDE in Equation \ref{eq: reverse SDE}, substituting the unconditional score \(\nabla_{\boldsymbol{x}_t}\log{p(\boldsymbol{x}_t)}\) with the approximated posterior score \(\nabla_{\boldsymbol{x}_t}\log p(\boldsymbol{x}_t|\boldsymbol{y})=\boldsymbol{s}_{\boldsymbol{\theta}}(\boldsymbol{x}_t,t)+\nabla_{\boldsymbol{x}_t}\log{\tilde{p}}(\boldsymbol{y}|\hat{\boldsymbol{x}}_{0}(\boldsymbol{x}_t))\). While \cite{chung2023diffusion} and \cite{song2023loss} analyze the DPS likelihood error, their link to the resulting posterior is indirect. The following Lemma \ref{lemma: stability of DPS} bridges this gap by establishing the stability of the DPS solver.
\begin{lemma}
\label{lemma: stability of DPS}
Consider the posterior \(p_\delta(\cdot\mid\boldsymbol{y})\) induced by the DPS solver under assumptions~\ref{appendix assumption: Score Network Regularity}--\ref{appendix assumption: Score Network Accuracy} in Appendix \ref{appendix: Proof of main results} with the terminal stopping time \(\delta\). Suppose the likelihood is linear Gaussian
\(
\boldsymbol{y}\mid\boldsymbol{x}
\sim
\mathcal{N}(F\boldsymbol{x},\sigma_y^2\boldsymbol{I}).
\)
Then, for any finite noisy measurement \(\boldsymbol{y}\in\mathbb{R}^{d_y}\) and any \(\boldsymbol{x}\in\mathcal{X}\), let the corresponding noiseless measurement be \(\boldsymbol{y}_* = F\boldsymbol{x}\). There exists a finite constant \(C_{\mathrm{stable}}>0\) such that
\begin{align}
    \label{eq: stability}
    D_{\mathrm{KL}}
    \left(
    p_\delta(\cdot\mid\boldsymbol{y}_*)
    \,\Vert\,
    p_\delta(\cdot\mid\boldsymbol{y})
    \right)
    \le
    C_{\mathrm{stable}}
    \|\boldsymbol{y}_*-\boldsymbol{y}\|_2^2 .
\end{align}
\vspace{-15pt}
\end{lemma} 
The detailed proof is provided in Appendix \ref{appendix: Proof of main results}. Lemma \ref{lemma: stability of DPS} indicates that the sensitivity is bounded by a term quadratic in the measurement noise magnitude $\|\boldsymbol{y}_*-\boldsymbol{y}\|_2$. This result confirms the stability of the diffusion prior under the DPS likelihood approximation, validating its effectiveness for moderate noise levels that align with the assumed Gaussian likelihood model \citep{chung2023diffusion}. However, since the bound depends on the measurements, it may grow substantially under outlier contamination, as observed empirically in \Cref{fig:error_with_contamination}. Hence,  the result establishes stability but does not provide a robustness guaranty against outlier contamination, motivating our robust posterior guidance.

Before proceeding, we first provide the definition of robustness by quantifying how the posterior distribution changes in response to measurement perturbations. Following \cite{altamirano2023robust,duran2024outlier}, the posterior influence function (PIF) is defined as the KL divergence such that
\(
\text{PIF}_{\boldsymbol{y}_*}(\boldsymbol{y})\coloneq D_{\text{KL}}\left(p(\boldsymbol{x}|\boldsymbol{y}_*)\| p(\boldsymbol{x|\boldsymbol{y}})\right).
\)
The posterior is called \textit{outlier-robust} if the PIF is uniformly bounded, i.e., \(\sup_{\boldsymbol{y}\in\mathbb{R}^{d_y}}|\text{PIF}_{\boldsymbol{y}_*}(\boldsymbol{y})|<\infty\).

\subsection{Robust Diffusion Posterior}
\label{subsec:RDP}
To address this issue, we consider generalized Bayes (GB) \citep{sullivan2015introduction} in this work. GB considers alternative differentiable loss functions $\ell_{\boldsymbol{y}}:\mathcal{X}\rightarrow\mathbb{R}$ in place of the negative log-likelihood in Bayes' rule. The subscript $\boldsymbol{y}$ denotes the dependence on the observation $\boldsymbol{y}$. The loss $\ell_{\boldsymbol{y}}$ can be flexibly specified under mild regularity conditions; see Section 6.3 of \cite{sullivan2015introduction}. This yields a generalized posterior \citep{matsubara2022robust, wild2023rigorous, altamirano2023robust} in the following form: \(p^{(\ell)}(\boldsymbol{x}|\boldsymbol{y})\propto p(\boldsymbol{x})\exp{\{-\tau\,\ell_{\boldsymbol{y}}(\boldsymbol{x})\}}\)
where $\tau>0$ is a temperature parameter that controls the influence of the loss relative to the prior. Many choices of $\ell$ (e.g., density-power) are well established and show improved performance under likelihood misspecification \citep{huber2011robust,song2024robust}. 

Such settings assume multiple observations, with robustness defined in terms of the contamination of only part of the dataset. Differently, BIPs focus on the posterior conditioning on a \emph{single} (high-dimensional) measurement, where contamination is structured and often occurs at the level of coordinates or components, e.g., pixel-wise outliers in images; see Figure~\ref{fig: demo}(b)--(c), thereby making standard GB constructions nontrivial to apply in this setting. 
\begin{wrapfigure}{l}{0.42\textwidth}
    \centering
    \vspace{1em}
    \includegraphics[width=0.4\textwidth]{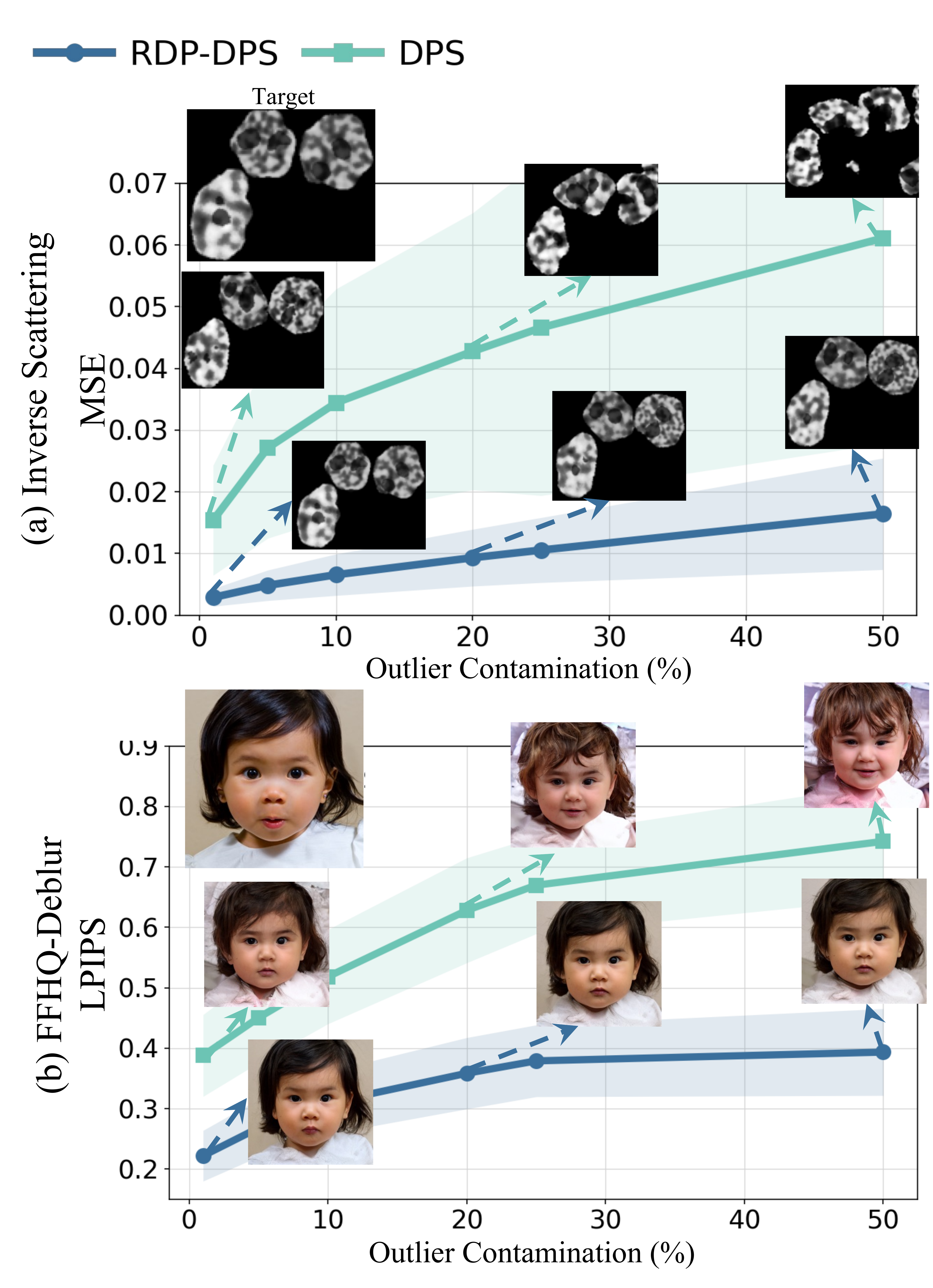}
    \caption{\small Recovery performances of DPS and RDP-DPS under increasing outlier contamination for inverse scattering and FFHQ-Deblur.}
    \label{fig:error_with_contamination}
    \vspace{-3em}
\end{wrapfigure}

To gain robustness in BIPs, we therefore choose $\ell_{\boldsymbol{y}}$ as a weighted likelihood loss and refer to the resulting method as the \textbackslash{}emph\{Robust Diffusion Posterior\} (RDP).  Following the additive likelihood form in \Cref{subsec: BIP}, RDP guides the reverse diffusion using a robust weighted likelihood evaluated at the DPS denoised estimate \(\hat{\boldsymbol{x}}_{0}(\boldsymbol{x}_t)\). RDP achieves robustness by transforming the DPS guidance term into a component-wise weighted objective. At diffusion time $t$, we define the loss function \(\ell_{\boldsymbol{y}}(\boldsymbol{x}_t)=-\;\sum_{i=1}^{d_y} w(r_{t,i})\,\log{\tilde{p}(y_i|\hat{\boldsymbol{x}}_{0|t})}\), where \(r_{t,i} = y_i - \hat{y}_{t,i}\) and \(\hat{y}_{t,i} \coloneq [F\hat{\boldsymbol{x}}_{0|t}]_i\) denote the fitting residuals and the predicted measurements, respectively. The loss function induces the robust diffusion posterior $p_t^{(\ell)}(\boldsymbol{x}_t|\boldsymbol{y})\propto p(\boldsymbol{x}_t)\exp{\{-\tau\,\ell_{\boldsymbol{y}}(\hat{\boldsymbol{x}}_{0|t})\}}$ at time $t$. The term $w\colon\mathbb{R}\rightarrow\mathbb{R}$ is a component-wise adaptive weight function. This weighting mechanism acts as a robust filter that given small residuals, $w_{i}\approx 1$, the loss function $\ell_{\boldsymbol{y}}$ recovers standard likelihood; conversely, large residuals are down-weighted to control outliers impact. 

We now provide a sufficient condition for the design of such a weighting function \(w\) to guarantee the diffusion posterior robustness, formalized in the following result:\\
\begin{theorem} 
\label{theorem: outlier-robust}
Consider the robust diffusion posterior with the score $\nabla_{\boldsymbol{x}_t}\log p_t^{(\ell)}(\boldsymbol{x}_t|\boldsymbol{y}) = \boldsymbol{s}_{\boldsymbol{\theta}}(\boldsymbol{x}_t, t) - \nabla_{\boldsymbol{x}_t}\ell_{\boldsymbol{y}}(\boldsymbol{x}_t)$. Let the guidance loss be defined as \(\ell_{\boldsymbol{y}}(\boldsymbol{x}_t)=-\sum_{i=1}^{d_y} w(r_i)\log p(y_i|\hat{\boldsymbol{x}}_{0|t})\), where \(r_i = y_i - [F(\hat{\boldsymbol{x}}_{0|t})]_i\) denotes the residual and \(w(\cdot)\) is a weighting function. Suppose assumptions~\ref{appendix assumption: Score Network Regularity}--\ref{appendix assumption: Score Network Accuracy} hold, if \(w\) satisfies the following conditions for any \(r\in\mathbb{R}\): 
\begin{equation} 
\label{eq: robust_weight_condition} |r\,w(r)| <\infty \quad \text{and} \quad |r^2w'(r)|<\infty, 
\end{equation}
where $w'$ is the gradient of $w$ w.r.t. $r$, the posterior KL divergence is uniformly bounded: 
\begin{talign*}
\sup_{\boldsymbol{y}\in\mathbb{R}^{d_y}} D_{\mathrm{KL}}\left( p_\delta^{(\ell)}\,\cdot\,|\boldsymbol{y}_*)\| p_\delta^{(\ell)}(\,\cdot\,|\boldsymbol{y})\right) \le C, 
\end{talign*} 
where \(C\) is a finite constant independent of the measurement \(\boldsymbol{y}\). Hence, the posterior \(p_\delta^{(\ell)}\) is outlier-robust, satisfying \(\sup_{\boldsymbol{y}\in\mathbb{R}^{d_y}}|\text{PIF}_{\boldsymbol{y}_*}(\boldsymbol{y})|<\infty\).
\end{theorem}
We provide the proof in Appendix~\ref{subsec appendix: Proof of Theorem outlier-robust}. This result stands in contrast to the standard DPS solver, whose PIF is unbounded. The proposed robust conditions control the influence of severe noise misspecification in $\boldsymbol{y}$, mitigating the risk of destabilizing the generation process.

\textbf{Choice of $w$.} The inverse multi-quadratic (IMQ) weighting function, defined as \(w(r) = (1 + r^2/c^2)^{-\nicefrac{1}{2}}\), satisfies the conditions established in Theorem \ref{theorem: outlier-robust}. Consequently, the resulting diffusion posterior is outlier-robust. In this work, we adopt $w$ in the IMQ form
\(w(r_i)=\left(1+r_i^2/c^2\right)^{-\nicefrac{1}{2}}\) for $i=1,...,d_y$, where $c>0$ is a soft threshold: a larger value of $c$ keeps weights closer to uniform, while a smaller value strongly down-weights large residuals. See Appendix \ref{Appendix subsec: Other Weight function choices Discussion} for additional discussion on alternative weight functions.

Algorithm~\ref{alg:rdps} presents RDP instantiated within the DPS sampler, which provides a realization aligned with the analysis in this paper. More broadly, RDP is a modular plug-in that can be incorporated into a broad class of plug-and-play diffusion-based approaches. We outline a general version in Appendix~\ref{Appendix subsec: General Pseudo Algorithm} Algorithm~\ref{alg: general robust diffusion posterior sampling} and provide associated experimental results in \Cref{sec: experiments}.

\paragraph{Performance given well-specification measurements.}
\begin{figure}[t]
\vspace{-1.5em}
    \centering
    \includegraphics[width=1\linewidth]{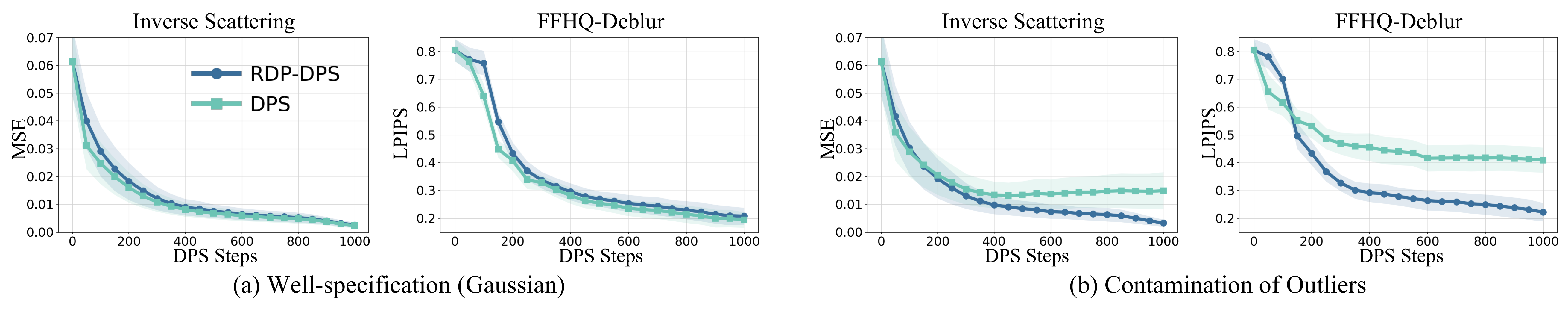}
    \caption{\small Recovery performance of DPS and RDP-DPS (IMQ) across denoising steps under Gaussian noise and outlier contamination, averaged over 100 cases per task.}
    \label{fig: Recovery performance of DPS and RDP-DPS}
    \vspace{-1.2em}
\end{figure}

\begin{wrapfigure}{r}{0.49\columnwidth}
\captionsetup{type=algorithm}
\hrule height 0.9pt
\vspace{0.2em}
\caption{Robust diffusion posterior sampling}
\label{alg:rdps}
\vspace{-0.5em}
\hrule height 0.5pt
\begin{algorithmic}[1]
\footnotesize
\STATE \textbf{Input:} observation \(\boldsymbol{y}\), score network \(\boldsymbol{s}_{\boldsymbol{\theta}}\), noise schedule \(\{\beta_t\}_{t=0}^T\), weight function \(w(\cdot)\), temperature \(\{\tau_t\}_{t=0}^T\)
\STATE Initialize \(\boldsymbol{x}_T \sim \mathcal{N}(\boldsymbol{0},\boldsymbol{I})\)

\FOR{$t=T,\ldots,1$}
    \STATE \(\alpha_t \gets 1-\beta_t,\quad \bar{\alpha}_t \gets \prod_{i=1}^t \alpha_i\)
    \STATE \(\hat{\boldsymbol{x}}_{0|t} \gets
    \frac{1}{\sqrt{\bar{\alpha}_t}}
    \bigl(\boldsymbol{x}_t+(1-\bar{\alpha}_t)\boldsymbol{s}_{\boldsymbol{\theta}}(\boldsymbol{x}_t,t)\bigr)\)
    \STATE \(\hat{\boldsymbol{y}} \gets F\,\hat{\boldsymbol{x}}_{0|t}\), \(\boldsymbol{r} \gets \boldsymbol{y}-\hat{\boldsymbol{y}}\)
    \STATE \(\boldsymbol{W} \gets (w(r_1),\dots,w(r_m))\)
    \STATE \(\ell_{\boldsymbol{y}}(\boldsymbol{x}_t) \gets
    \boldsymbol{W}^\top \log \tilde{p}\bigl(\boldsymbol{y}\mid \hat{\boldsymbol{x}}_0(\boldsymbol{x}_t)\bigr)\)
    \STATE \(\hat{\boldsymbol{s}} \gets
    \boldsymbol{s}_{\boldsymbol{\theta}}(\boldsymbol{x}_t,t)
    -\tau_t\nabla_{\boldsymbol{x}_t}\ell_{\boldsymbol{y}}(\boldsymbol{x}_t)\)
    \STATE \(\boldsymbol{z}\sim\mathcal{N}(\boldsymbol{0},\boldsymbol{I})\) if \(t>1\), else \(\boldsymbol{0}\)
    \STATE \(\boldsymbol{x}_{t-1}\gets
    \boldsymbol{x}_t+\beta_t\left(\frac{1}{2}\boldsymbol{x}_t+\hat{\boldsymbol{s}}\right)\)
\ENDFOR

\STATE \textbf{Return:} \(\boldsymbol{x}_0\)
\end{algorithmic}
\vspace{0.2em}
\hrule height 0.9pt
\vspace{-1.0em}
\end{wrapfigure}

So far, we have focused on the robustness of RDP under outlier-contaminated measurements. A natural follow-up question is whether this robustness comes at the cost of degraded performance given well-specified measurements. In this setting, the standard Gaussian likelihood provides the desired guidance. The robust weighted likelihood introduces a bias relative to this Gaussian guidance, and the magnitude of this bias depends on the threshold parameter $c$. For the IMQ weighting function, \(w(r)=(1+r^2/c^2)^{-1/2}\).  When the residual $r$ is small relative to $c$,  $w(r)\approx1$ that roughly recovers the Gaussian guidance. This is analogous to the discussion in standard Bayesian problems \citep{knoblauch2022optimization}. A larger $c$ favors high fidelity for clean data (low bias), while a smaller $c$ provides stronger suppression of large fitting residuals (high robustness).

\begin{wrapfigure}{r}{0.47\textwidth}
    \centering
    \vspace{-0.5em}
    \includegraphics[width=0.97\linewidth]{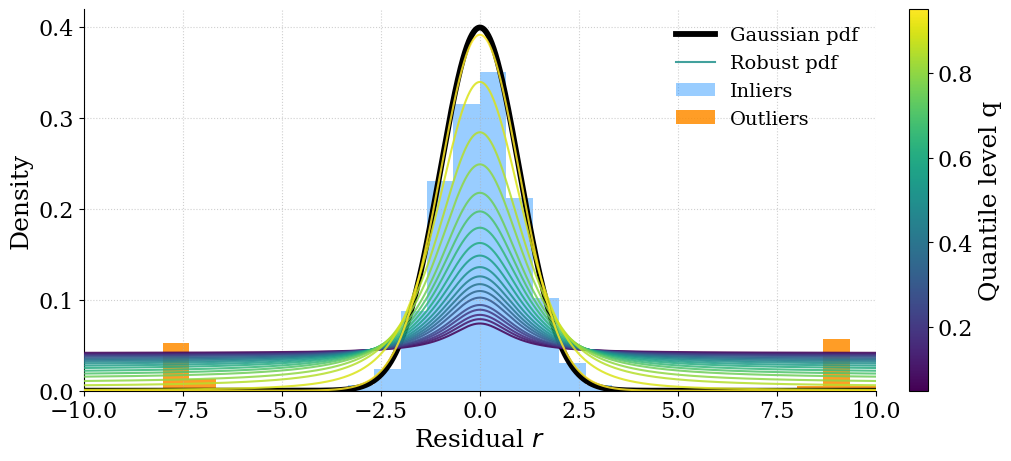}
    \caption{\small Histogram of residuals and outliers compared against Gaussian (black) and robust pdfs (colored) across quantiles $q$.}
    \label{fig: robust_pdf}
    \vspace{-0.35em}
\end{wrapfigure}

Moreover, the denoising process introduces a time-dependent residual scale, which typically decreases as the denoised estimate becomes more accurate. This makes a fixed threshold \(c\) less appropriate. In practice, we adopt an adaptive schedule for the parameter \(c\) to account for these dynamic changes over time. We dynamically calibrate $c_t$ at each timestep $t$ based on the statistics of the current fitting residuals. Specifically, we set \(c_t\) based on the \(q\)-th quantile of the absolute residuals:
\(
c_t = \text{Quantile}_q(|\boldsymbol{r}_t|),
\)
where \(q \in (0, 1]\) and \(\boldsymbol{r}_t=(r_{t,1},\ldots,r_{t,d_y})\).
By setting \(c\) at the $q$-th quantile $(\%)$, the posterior update is primarily driven by the $q\%$ of measurements with smaller residual magnitudes, while controlling the impact of the remaining $(1-q)\%$ as presented in Figure \ref{fig: robust_pdf}. With the adaptive design of \(c\), RDP-DPS closely matches DPS under well-specified measurements and provides more stable recovery under outlier contamination. We further include the empirical validation and  additional ablations in Section~\ref{subsec: Result Analysis}.

\begin{table*}[ht]
\centering
\caption{\small Quantitative results on \textit{Inverse Scattering} are reported as PSNR, SSIM, and NMAE (mean $\pm$ std) over 100 test samples. Results on \textit{FFHQ-Inpaint}, \textit{-Deblur}, and \textit{-Phase Retrieval}, reported in LPIPS~($\downarrow$), VIFL~($\downarrow$) as mean~$\pm$~std, and FID~($\downarrow$) on 500 test images.}
\label{tab: overall experiments}
\resizebox{\textwidth}{!}{%
\begin{tabular}{clccccccccc}
\toprule
\multirow{2}{*}{\textbf{Task}} & \multirow{2}{*}{\textbf{Method}}
  & \multicolumn{3}{c}{\textbf{Well-specification (Gaussian)}} 
  & \multicolumn{3}{c}{\textbf{Noise Misspecification (Student-T)}} 
  & \multicolumn{3}{c}{\textbf{Contamination of Outliers}} \\
\cmidrule(lr){3-5}\cmidrule(lr){6-8}\cmidrule(lr){9-11}
 & & PSNR($\uparrow$) & SSIM($\uparrow$) & NMAE($\downarrow$) & PSNR($\uparrow$) & SSIM($\uparrow$) & NMAE($\downarrow$) & PSNR($\uparrow$) & SSIM($\uparrow$) & NMAE($\downarrow$) \\
\midrule
\multirow{8}{*}{\rotatebox{90}{\shortstack{\textbf{Inverse}\\\textbf{Scattering}}}}
& DPS     & $26.26\pm2.83$ & $0.86\pm0.02$ & $0.16\pm0.02$ & $21.53\pm2.93$ & $0.78\pm0.06$ & $0.25\pm0.05$ & $19.37\pm3.44$ & $0.71\pm0.10$ & $0.33\pm0.05$ \\
& LGD     & $24.16\pm2.95$ & $0.76\pm0.20$ & $0.18\pm0.09$ & $18.80\pm2.54$ & $0.62\pm0.07$ & $0.34\pm0.10$ & $17.81\pm2.89$ & $0.53\pm0.12$ & $0.36\pm0.05$ \\
& $\Pi$GDM     & $24.23\pm2.26$ & $0.81\pm0.13$ & $0.18\pm0.07$ & $21.17\pm3.15$ & $0.69\pm0.05$ & $0.24\pm0.05$
        & $15.14\pm3.88$ & $0.64\pm0.15$ & $0.37\pm0.06$ \\
& PnPDM & $25.77\pm2.75$ & $0.87\pm0.04$ & $\mathbf{0.14\pm0.03}$ & $20.38\pm4.14$ & $0.73\pm0.12$ & $0.27\pm0.21$ & $15.33\pm3.23$ & $0.73\pm0.12$ & $0.30\pm0.03$ \\
& DiffPIR & $\mathbf{27.36\pm2.79}$ & $\mathbf{0.88\pm0.07}$ & $0.15\pm0.02$ & $22.91\pm3.33$ & $0.80\pm0.10$ & $0.23\pm0.03$ & $19.08\pm3.92$ & $0.71\pm0.13$ & $0.33\pm0.08$\\
\addlinespace[2pt]
\cdashline{2-11}
\addlinespace[2pt]
& RDP-DPS   &  $26.01\pm2.72$    &    $0.86\pm0.02$    &    $0.15\pm0.01$           &  $\mathbf{23.84\pm2.59}$      &   $\mathbf{0.82\pm0.03}$          &        $\mathbf{0.19\pm0.04}$       &        $\mathbf{25.24\pm2.74}$        &    $\mathbf{0.83\pm0.02}$       &    $\mathbf{0.15\pm0.02}$         \\
& RDP-LGD  & $24.88\pm2.17$   & $0.81\pm0.06$    &    $0.18\pm0.03$           &  $21.83\pm2.25$     &    $0.79\pm0.05$      &    $0.23\pm0.04$      &      $25.02\pm2.23$  &    $0.70\pm0.06$  & $0.21\pm0.03$ \\
& RDP-$\Pi$GDM   &  $23.67\pm2.15$    &   $0.79\pm0.07$   & $0.17\pm0.04$
         & $22.23\pm2.18$        &   $0.72\pm0.02$   & $0.19\pm0.03$
         &  $22.49\pm2.36$       &   $0.78\pm0.05$   &  $0.21\pm0.04$\\
& RDP-PnPDM & $23.94\pm2.65$ & $0.86\pm0.04$ & $0.16\pm0.02$ & $22.15\pm3.11$ & $0.79\pm0.06$ & $0.21\pm0.02$ & $23.45\pm2.95$ & $0.81\pm0.07$ & $0.16\pm0.03$ \\
\midrule
 & & LPIPS($\downarrow$) & VIFL($\downarrow$) & FID($\downarrow$)
   & LPIPS($\downarrow$) & VIFL($\downarrow$) & FID($\downarrow$)
   & LPIPS($\downarrow$) & VIFL($\downarrow$) & FID($\downarrow$) \\

\midrule

\multirow{8}{*}{\raisebox{-0.5\height}{\rotatebox[origin=c]{90}{\textbf{Inpaint}}}}
& DPS      & $\mathbf{0.19\pm0.04}$ & $0.62\pm0.04$ & $27.19$ & $0.25\pm0.06$ & $0.70\pm0.04$ & $30.21$ & $0.37\pm0.06$ & $0.89\pm0.04$ & $45.27$ \\
& LGD      & $0.21\pm0.03$ & $0.68\pm0.04$ & $28.04$ & $0.24\pm0.04$ & $0.75\pm0.05$ & $31.79$ & $0.34\pm0.06$ & $0.82\pm0.03$ & $37.93$ \\
& $\Pi$GDM & $0.21\pm0.05$ & $0.63\pm0.03$ & $28.72$ & $0.29\pm0.03$ & $0.77\pm0.04$ & $40.28$  & $0.38\pm0.06$ & $0.84\pm0.04$ & $47.93$ \\
& PnPDM    & $\mathbf{0.19\pm0.03}$ & $\mathbf{0.61\pm0.04}$ & $\mathbf{27.17}$ & $0.25\pm0.04$ & $0.73\pm0.04$ & $31.29$ & $0.36\pm0.04$ & $0.85\pm0.05$ & $47.28$ \\
& DiffPIR & $0.24\pm0.04$ & $0.70\pm0.06$ & $29.64$ & $0.26\pm0.07$ & $0.77\pm0.05$ & $35.83$ & $0.37\pm0.06$ & $0.87\pm0.04$ & $46.80$\\
\addlinespace[2pt]
\cdashline{2-11}
\addlinespace[2pt]
& RDP-DPS  & $0.20\pm0.04$ & $0.62\pm0.05$ & $27.62$ & $\mathbf{0.21\pm0.04}$ & $\mathbf{0.66\pm0.04}$ & $\mathbf{28.03}$ & $\mathbf{0.24\pm0.04}$ & $\mathbf{0.70\pm0.03}$ & $\mathbf{29.95}$ \\
& RDP-LGD  & $0.22\pm0.04$ & $0.69\pm0.04$ & $28.60$ & $0.23\pm0.05$ & $0.72\pm0.03$ & $29.52$ & $0.25\pm0.05$ & $0.77\pm0.04$ & $30.06$ \\
& RDP-$\Pi$GDM & $0.20\pm0.04$ & $0.63\pm0.03$ & $29.15$ & $0.24\pm0.04$ & $0.70\pm0.04$ & $28.19$ & $0.25\pm0.04$ & $0.74\pm0.05$ & $30.27$ \\
& RDP-PnPDM & $0.20\pm0.04$ & $0.65\pm0.03$ & $27.66$ & $0.23\pm0.04$ & $0.71\pm0.03$ & $29.35$ & $\mathbf{0.24\pm0.04}$ & $0.73\pm0.04$ & $30.20$ \\
\midrule
\multirow{8}{*}{\raisebox{-0.5\height}{\rotatebox[origin=c]{90}{\textbf{Deblur}}}}
& DPS      & $0.17\pm0.04$ & $0.61\pm0.05$ & $26.98$ & $0.20\pm0.04$ & $0.69\pm0.04$ & $29.35$ & $0.32\pm0.07$ & $0.83\pm0.07$ & $39.82$ \\
& LGD      & $\mathbf{0.13\pm0.03}$ & $\mathbf{0.58\pm0.04}$ & $\mathbf{22.36}$ & $0.17\pm0.04$ & $0.67\pm0.05$ & $22.84$ & $0.33\pm0.09$ & $0.79\pm0.08$ & $33.39$ \\
& $\Pi$GDM & $0.19\pm0.08$ & $0.62\pm0.06$ & $30.77$ & $0.29\pm0.11$ & $0.71\pm0.14$ & $55.76$ & $0.39\pm0.15$ & $0.90\pm0.18$ & $59.32$ \\
& PnPDM    & $0.14\pm0.04$ & $0.59\pm0.04$ & $24.11$ & $0.21\pm0.05$ & $0.70\pm0.04$ & $35.18$ & $0.35\pm0.04$ & $0.84\pm0.04$ & $42.54$ \\
& DiffPIR & $0.16\pm0.04$ & $0.61\pm0.05$ & $25.79$ & $0.22\pm0.04$ & $0.73\pm0.06$ & $30.86$ & $0.38\pm0.06$ & $0.88\pm0.07$ & $54.78$\\
\addlinespace[2pt]
\cdashline{2-11}
\addlinespace[2pt]
& RDP-DPS  & $0.17\pm0.04$ & $0.62\pm0.04$ & $26.39$ & $0.18\pm0.04$ & $0.65\pm0.03$ & $27.81$ & $0.25\pm0.06$ & $0.75\pm0.03$ & $30.60$ \\
& RDP-LGD  & $0.14\pm0.05$ & $0.61\pm0.05$ & $22.84$ & $\mathbf{0.14\pm0.03}$ & $\mathbf{0.63\pm0.04}$ & $\mathbf{22.71}$ & $\mathbf{0.22\pm0.05}$ & $\mathbf{0.70\pm0.04}$ & $\mathbf{28.55}$ \\
& RDP-$\Pi$GDM & $0.20\pm0.06$ & $0.63\pm0.05$ & $30.86$ & $0.21\pm0.04$ & $0.69\pm0.05$ & $33.94$ & $0.31\pm0.06$ & $0.83\pm0.07$ & $38.84$ \\
& RDP-PnPDM & $0.16\pm0.04$ & $0.61\pm0.03$ & $25.73$ & $0.18\pm0.04$ & $0.65\pm0.04$ & $27.69$ & $0.25\pm0.05$ & $0.77\pm0.04$ & $33.92$ \\
\midrule
\multirow{6}{*}{\rotatebox{90}{\shortstack{\textbf{Phase}\\\textbf{Retrieval}}}}
& DPS      & $0.45\pm0.17$ & $0.76\pm0.03$ & $62.16$ & $0.52\pm0.18$ & $0.85\pm0.03$ & $76.01$ & $0.64\pm0.22$ & $0.96\pm0.02$ & $118.59$ \\
& LGD      & $0.43\pm0.19$ & $0.79\pm0.04$ & $64.86$ & $0.55\pm0.21$ & $0.89\pm0.04$ & $95.42$ & $0.60\pm0.17$ & $0.95\pm0.05$ & $140.72$ \\
& PnPDM    & $\mathbf{0.38\pm0.13}$ & $\mathbf{0.72\pm0.04}$ & $\mathbf{57.04}$ & $0.49\pm0.17$ & $0.82\pm0.03$ & $74.38$ & $0.60\pm0.20$ & $0.93\pm0.04$ & $113.82$ \\
\addlinespace[2pt]
\cdashline{2-11}
\addlinespace[2pt]
& RDP-DPS  & $0.43\pm0.13$ & $0.76\pm0.02$ & $61.29$ & $0.49\pm0.12$ & $0.81\pm0.02$ & $73.51$ & $\mathbf{0.51\pm0.14}$ & $\mathbf{0.82\pm0.04}$ & $\mathbf{75.07}$ \\
& RDP-LGD  & $0.43\pm0.16$ & $0.79\pm0.04$ & $63.37$ & $0.51\pm0.14$ & $0.83\pm0.05$ & $76.28$ & $0.54\pm0.09$ & $0.88\pm0.03$ & $84.27$ \\
& RDP-PnPDM & $0.40\pm0.12$ & $0.74\pm0.02$ & $58.93$ & $\mathbf{0.47\pm0.13}$ & $\mathbf{0.77\pm0.04}$ & $\mathbf{63.93}$ & $0.52\pm0.16$ & $\mathbf{0.82\pm0.04}$ & $75.31$ \\
\bottomrule
\end{tabular}
}
\vspace{-0.5em}
\end{table*}

\section{Experiment}
\label{sec: experiments}
In this section, we validate the effectiveness and robustness of the proposed RDP method through extensive experiments and ablation studies. We demonstrate RDP's generality by integrating it with four representative diffusion-based inverse solvers, including DPS \citep{chung2023diffusion}, LGD \citep{song2023loss}, $\Pi$GDM \citep{song2023pseudoinverse}, and PnPDM \citep{wu2024principled}\footnote{PnPDM does not rely on the Tweedie approximation and therefore lies outside the formal analysis of Sec.~\ref{sec: method}. However, it also employs gradient-based guidance; adding RDP yields improved performance.}. We also include DiffPIR~\citep{zhu2023denoising} as another representative diffusion restoration approach. See \Cref{alg:rdps} and \Cref{alg: general robust diffusion posterior sampling} for the pseudo-code of our algorithms. For fair comparisons, all methods are evaluated using the same pretrained score/denoising models and sampling schedule. We refer to the RDP version of method X as RDP-X; for example, RDP-DPS denotes DPS with RDP guidance. Implementation code is available in the supplementary material. See additional experimental results and more ablation studies in \Cref{appendix: additional experimental results}.

\subsection{Experiment Settings}
\paragraph{Tasks.} We conduct experiments on a diverse range of scientific inverse problems and natural image restorations using the open-source benchmark \textit{InverseBench}\footnote{\url{https://github.com/devzhk/InverseBench}} \citep{zheng2025inversebench}. The tasks include: (1) inverse scattering (IS), (2) image inpainting, (3) deblurring, and (4) phase retrieval on the FFHQ-256 dataset (FFHQ-D/PR), covering both linear and nonlinear cases. In practice, these problems often suffer from likelihood misspecification; for example, phase retrieval can be affected by unexpected scattering from inhomogeneous media and other unmodeled physical corruptions. Such effects can introduce heavy-tailed noise or outlier contamination that violate Gaussian assumptions, making these tasks well suited for evaluating multi-perspective robust reconstructions.
\begin{figure*}[t]
    \centering
    \includegraphics[width=1\textwidth]{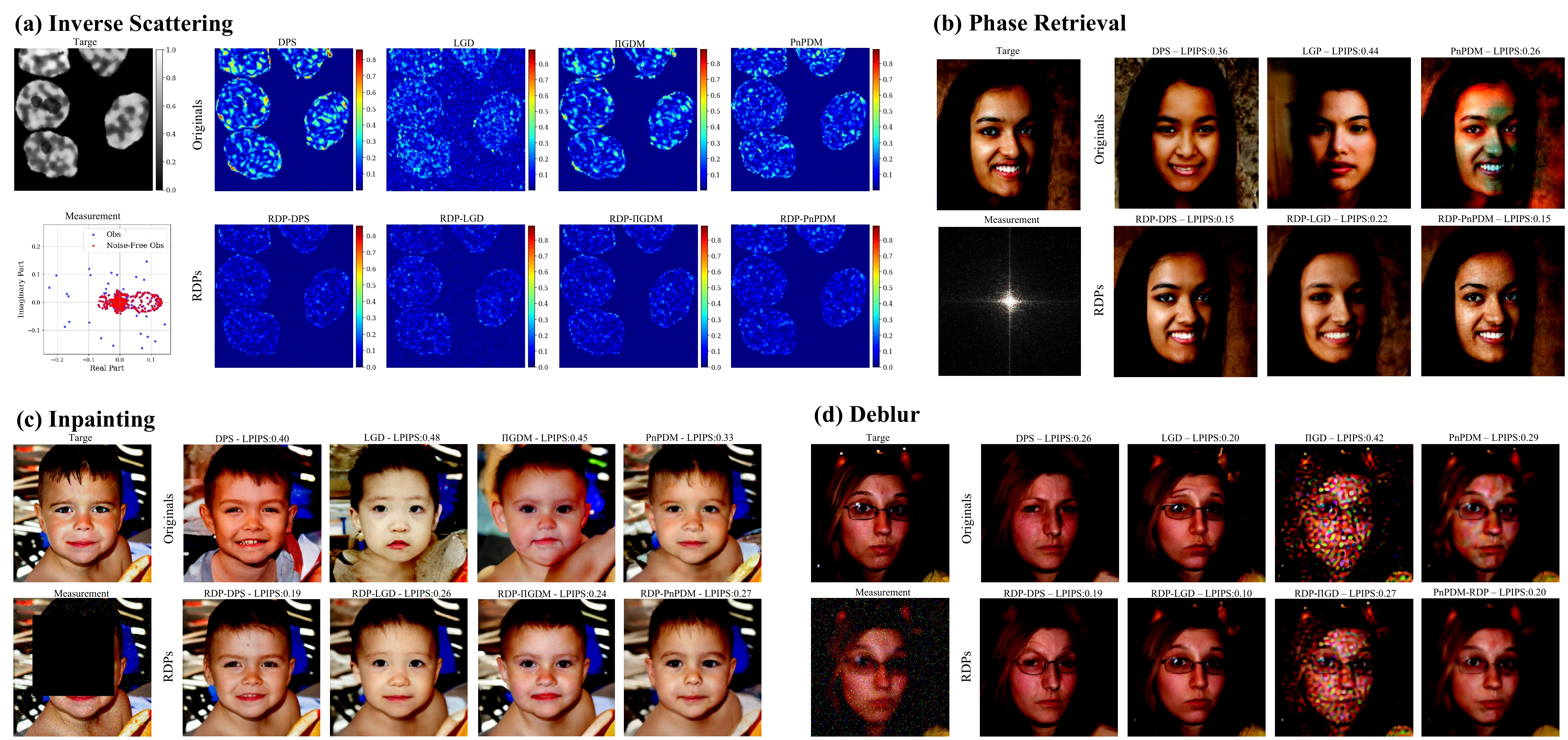} 
        \caption{\small Qualitative results on all outlier-corrupted tasks: (a) Inverse Scattering, (b) Phase Retrieval, (c) Deblurring, and (d) Inpainting. In each panel, the leftmost column shows the ground truth and the measurements. The remaining columns compare the originals (top row) with their robust versions (bottom row); we report reconstruction absolute errors in (a) and reconstructed samples in (b)–(d).}
    \label{fig: combined_qual_figure_outlier}
\end{figure*}
\begin{figure*}
    \centering
    \includegraphics[width=1\linewidth]{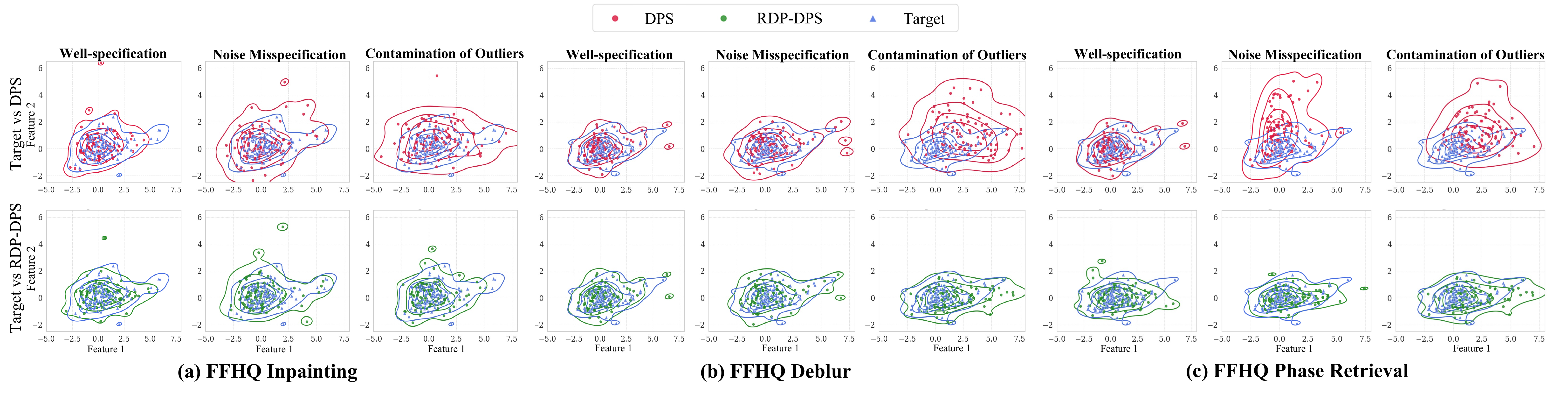}
    \caption{Posterior visualizations in Inception-v3 latent space for (a) Inpainting, (b) Deblur, and (c) Phase Retrieval under well-specified, misspecified, and outlier-contaminated settings. Top: DPS (\textcolor{red}{$\bullet$}); bottom: RDP-DPS (\textcolor{green}{$\bullet$}); target: \textcolor{blue}{$\blacktriangle$}. The density contours are estimated by kernel density estimation.}
    \label{fig: latentspaceall}
    \vspace{-0.35em}
\end{figure*}
\paragraph{Experimental Design.} A central focus of our evaluation is the performance regarding the mismatched noise distribution. We introduce two different schemes applied in the observation domain: (1) \emph{Heavy-tailed noise}: i.i.d. noises from the Student-t distribution for evaluating the performance under the misspecification of noise; and (2) \emph{Sparse impulsive outlier}: a fraction of components/pixels corrupted by large-magnitude perturbations for evaluating the performance under outlier contamination. Additionally, to evaluate our approaches under normal measurements, we also report results with the conventional i.i.d. Gaussian noise distribution. \Cref{subsec: Noise models} includes specific settings for each experiment.

\paragraph{Evaluation Metric.} For the IS problem, we quantify the reconstruction quality using peak signal-to-noise ratio (PSNR), structural similarity (SSIM), and the normalized mean absolute error (NMAE) metrics. In the two natural image tasks, we use standard perceptual quality metrics, including learned perceptual image patch similarity (LPIPS), visual information fidelity loss (VIFL), and Fr\'echet Inception Distance (FID) using the pre-trained InceptionV3 \citep{szegedy2016rethinking}. 

\subsection{Result Analysis} 
\label{subsec: Result Analysis}
In this section, we answer the following questions with key findings from our experiments.

(1) \emph{When measurements match the presumed likelihood model, a diffusion-model prior lead to posterior stability and convergence.} 

Quantitatively, Table~\ref{tab: overall experiments} shows that, when the Gaussian likelihood is well specified, RDP performs competitive with standard methods such as DPS, LGD, \(\Pi\)GDM, and PnPDM, indicating that the robust likelihood approximation preserves stable convergence. This is also supported by Figure~\ref{fig: latentspaceall} (left columns), where the sample and target density contours largely overlap in the first two latent dimensions. Figure~\ref{fig: Recovery performance of DPS and RDP-DPS} further shows that, with the adaptive choice of \(c_t\), RDP-DPS closely tracks DPS over diffusion time in the well-specified setting, while achieving more robust recovery under outlier contamination. Additional ablations in Section~\ref{subsec: Result Analysis} also show that the method is not highly sensitive to the choice of \(q\).

(2) \emph{Diffusion samplers can fail under misspecification,  and RDP effectively mitigate this failure.} 

\textsc{(Degradation of original approaches.)} In IS, the baselines in \Cref{tab: overall experiments} degrade by about $10\%-30\%$ due to the noise-distribution misspecification and $15\%-50\%$ under outlier contamination across metrics, compared with the well-specified setting. Original approaches typically rely on a presumed Gaussian likelihood and can degrade substantially under likelihood misspecification given non-Gaussian noise and outlier contamination.  Similar degradation phenomena are observed in image restoration tasks. Qualitatively in \Cref{fig: combined_qual_figure_outlier}, outliers induce visible artifacts, such as bright spots in the deblurring results of $\Pi$GDM and PnPDM , indicating limited robustness.

\textsc{(RDP Improvement.)} The integration of RDP substantially mitigates these issues and improves robustness to likelihood misspecification. In IS, RDP improves over baselines by roughly $5\%-29\%$ under noise-distribution misspecification and $20\%-50\%$ under outlier contamination across metrics. Similar gains are observed on FFHQ-D and FFHQ-PR in Table~\ref{tab: overall experiments}. Qualitatively, RDP removes bright-spot artifacts in $\Pi$GDM and PnPDM variants, as shown in Figures~\ref{fig: combined_qual_figure_outlier} and~\ref{fig: combined_qual_figure_student_t}, consistent with its adaptive down-weighting of extreme residuals. Evaluated in the Inception-v3 latent space (Figure~\ref{fig: latentspaceall}, middle/right), RDP-DPS remains closer to the target distribution than DPS under misspecification.

RDP consistently improves performance in both scenarios, but its gains are larger under outlier contamination than under heavy-tailed noise misspecification. We attribute this difference to the fact that large outliers are easier to identify and down-weight, whereas heavy-tailed noise is less separable from normal measurement noise.

(3) \emph{RDP are generalizable in (i) integration into  different diffusion samplers and (ii) in validation across different tasks.} 

This generality follows from our design: the robust gradient-based guidance term is decoupled from the underlying diffusion solver, making RDP a sampler-agnostic module that can be integrated into different gradient-based posterior samplers. Across image restoration and scientific tasks, Table~\ref{tab: overall experiments} shows that RDP maintains competitive performance under well-specified Gaussian noise while improving the original samplers under model misspecification. The qualitative results in Figures~\ref{fig: combined_qual_figure_outlier} and~\ref{fig: combined_qual_figure_student_t} further support these findings across tasks.

\subsection{Ablation Study}

We revisit the IS task to ablate RDP under different measurement conditions and design choices. Detailed settings and extensive additional results are provided in \Cref{appendix: additional experimental results}. We summarize the main findings below.

\textit{RDP is more beneficial when measurements are limited.} We first examine the performances of RDP under different measurement dimensions $d_y$. As shown in \Cref{tab:Performances under different noise settings and receiver configurations}, the advantage of RDP becomes more pronounced as the measurement dimension \(d_y\) decreases under both outlier and heavy-tailed noise. This suggests that robust guidance is particularly effective when measurements are limited, where each corrupted observation has a larger influence on the posterior update. 

\textit{RDP is robust to weighting designs.} We examine the sensitivity of RDP to the weighting function and the adaptive quantile \(q\). As shown in \Cref{tab:r2_weighting_ablation}, IMQ and Huber weighting achieve comparable performance under Gaussian and heavy-tailed noise, while IMQ gives slightly better results under outlier contamination. This suggests that the key factor is not the specific weighting form, but the adaptive down-weighting of large residuals. Since both IMQ and Huber use quantile-calibrated thresholds, they suppress a similar set of unreliable measurements. \Cref{tab:r3_quantile_ablation} shows that the performance is stable for \(q\geq 0.25\), while an overly small quantile such as \(q=0.10\) degrades performance by aggressively down-weighting useful measurements. These results indicate that RDP is not highly sensitive to the precise weighting choice or quantile.

\section{Conclusion}
We studied the \emph{stability} and \emph{robustness} of Bayesian inverse problems with diffusion model priors, providing both theoretical insights and empirical evidence. Building on this analysis, we introduced \emph{robust diffusion posterior sampling} (RDP), a unified plug-and-play approach that can be readily incorporated into existing gradient-based diffusion posterior samplers. RDP preserves the stability of standard diffusion posterior sampling while providing robustness under outlier contamination. Extensive experiments across scientific inverse problems and natural image restoration tasks validate our theory and demonstrate the effectiveness and robustness of the proposed method.

\paragraph{Limitations.}
This work has two main limitations. First, our theoretical analysis mainly focuses on linear BIPs, and extending it to nonlinear settings remains an important direction. Second, our experiments primarily consider isotropic noise, while structured noise, such as correlated or spatially varying noise, is also common in practice. 

\newpage

\bibliographystyle{plainnat}
\bibliography{references}

\appendix
\clearpage
\section{Notation}
\label{appsec: notation}

\begin{table}[h!]
\centering
\small
\caption{Summary of Key Notations.}
\vspace*{3mm}
\begin{tabular}{l | l}
\noalign{\hrule}
\noalign{\hrule}
\textbf{Notation} & \textbf{Meaning} \\
\noalign{\hrule}
\multicolumn{2}{c}{\textit{Inverse Problem}} \\
\noalign{\hrule}
$\boldsymbol{x}$ & target variable \\
$\boldsymbol{y}_*$ & noise-free measurement \\
$\boldsymbol{y}$ & observed measurement \\
$F(\cdot)$ & Measurement model \\
$\boldsymbol{\epsilon}$ & Measurement noise ($\boldsymbol{y} = F(\boldsymbol{x}) + \boldsymbol{\epsilon}$) \\

\noalign{\hrule}
\multicolumn{2}{c}{\textit{Diffusion Model}} \\
\noalign{\hrule}

$p_{\text{data}}$ & True data distribution \\
$\boldsymbol{x}_t$ & Noised target variable at time $t \in [0, T]$ \\
$\boldsymbol{x}_0$ & Clean data samples \\
$\boldsymbol{s}_{\boldsymbol{\theta}}(\boldsymbol{x}_t, t)$ & Learned score function estimator \\
$\boldsymbol{x}_{0}(\boldsymbol{x}_t)$ & Ideal denoised estimate using the true score $\nabla \log p_t$; $\boldsymbol{x}_{0|t} = \mathbb{E}[\boldsymbol{x}_0 | \boldsymbol{x}_t]$ \\
$\hat{\boldsymbol{x}}_{0}(\boldsymbol{x}_t)$ & Predicted denoised estimate using the score network $\boldsymbol{s}_{\boldsymbol{\theta}}$; approximation of $\boldsymbol{x}_{0|t}$ \\
$\beta_t, \bar{\alpha}_t$ & Noise schedule and scaling factor \\
$p(\boldsymbol{x}_t|\boldsymbol{y})$ & intermediate posterior distribution at time $t$\\

\noalign{\hrule}
\multicolumn{2}{c}{\textit{Generalized Bayesian Inference}} \\
\noalign{\hrule}

$\ell_{\boldsymbol{y}}(\boldsymbol{x}_t)$ & Robust likelihood function \\
$\tau_t$ & temperature (learning rate) at time $t$ \\
$p^{(\ell)}(\boldsymbol{x}_t|\boldsymbol{y})$ & Generalized posterior distribution at time $t$ \\
$\hat{\boldsymbol{y}}_t$ & Predicted observation; $\hat{\boldsymbol{y}}_t = F(\hat{\boldsymbol{x}}_{0|t})$ \\
$\boldsymbol{r}_t$ & Residual vector; $\boldsymbol{r}_t = \boldsymbol{y} - \hat{\boldsymbol{y}}_t$ \\
$\omega_t$ & Learning rate/temperature at time $t$ \\
$W(\boldsymbol{r}_t)$ & Robust weighting matrix \\
PIF & Posterior Influence Function \\

\noalign{\hrule}
\multicolumn{2}{c}{\textit{General Notation}} \\
\noalign{\hrule}

$\mathbb{E}[\cdot]$ & Expectation \\
$\mathbb{V}[\cdot]$ & Variance (scalar) or Covariance matrix (vector) \\
$\nabla_{\boldsymbol{z}}$ & Gradient with respect to $\boldsymbol{z}$ \\
$\boldsymbol{I}$ & Identity matrix \\
$\|\cdot\|_p$ & $L_p$ norm \\
$J_f(\cdot)$ & Jacobian of function $f$ \\
$H_f(\cdot)$ & Hessian of function \(f\) \\
\noalign{\hrule}
\noalign{\hrule}
\end{tabular}
\label{table: notations}
\end{table}

\clearpage

\section{Review of Tweedie approximation in posterior sampling.}
\label{Appedix: Review of Tweedie's approximation in posterior sampling.}
\vspace{-0.5em} 
\paragraph{Conditional Score Estimation: The \textit{Plug-and-Play} Approach.} The most direct application of Tweedie’s approximation in posterior sampling is found in methods that attempt to estimate the conditional score $\nabla_{\boldsymbol{x}_t} \log p(\boldsymbol{y}|\boldsymbol{x}_t)$ and plug it into the reverse SDE. This plug-and-play (PnP) family, commonly referred to as Diffusion Posterior Sampling (DPS) \citep{graikos2022diffusion,chung2023diffusion}, replaces an intractable conditional expectation with a Tweedie-based point estimator. We focus our analysis on this class. Despite its empirical effectiveness, the approximation is heuristic and can induce systematic bias \citep{chung2023diffusion}. To mitigate this approximation error, other methods utilize the structure of (linear) measurement model to derive more principled guidance signals. For instance, DDRM \citep{kawar2022denoising} approaches linear inverse problems by performing diffusion in the spectral domain of the linear measurement model via SVD, while the $\Pi$GDM \citep{song2023pseudoinverse} employs the Moore-Penrose pseudoinverse to project the guidance signal directly onto the solution manifold consistent with the measurements. Moving beyond first-order point estimates, recent works incorporate second-order corrections. \cite{rozet2024learning} proposes matching the variance of $p(\boldsymbol{x}_0 | \boldsymbol{x}_t)$ by approximating it as a Gaussian rather than a Dirac delta. Similarly, Tweedie Moment Projected Diffusions (TMPD) \citep{boys2024tweedie} use higher-order Tweedie expansions to estimate the posterior covariance. Though statistically optimal for near-Gaussian distributions, this approach improves stability at the cost of computing the explicit Jacobian of the Tweedie estimator.

\vspace{-0.75em} 
\paragraph{Sequential Monte Carlo: Unbiased Estimator.} While the conditional score approach introduces systematic bias into the generation process due to Tweedie's approximation, an alternative line of research utilizes Sequential Monte Carlo (SMC) to correct these errors \citep{wu2023practical,janati2025bridging}. SMC methods, such as the Twisted Diffusion Sampler (TDS) \citep{wu2023practical} and Monte Carlo Guided Diffusion (MCGDiff) \citep{cardoso2024monte}, provide theoretically well-founded algorithms for posterior sampling. Unlike the DPS, SMC offers strong guaranties: the weighted particle approximation is asymptotically exact as the number of particles goes to infinity \citep{doucet2001sequential,douc2008limit}.  Recent work on the Radon-Nikodym Estimator (RNE) \citep{he2025rne} has further unified these approaches. RNE demonstrates that the heuristic weight functions used in methods like TDS actually correspond to the theoretically optimal weights derived from the ratio of path measures between the forward and backward processes. This theoretical unification confirms that Tweedie’s approximation acts as the crucial link that makes these rigorous likelihood-ratio weights computable in practice.  For linear inverse problems, \cite{song2021solving} introduces an auxiliary noise measurements sequence \(\{\boldsymbol{y}_t\}_t\) that is coupled to the latent diffusion trajectory \(\{\boldsymbol{x}_t\}_t\) through the same noise schedule. This coupling renders the intermediate likelihood $p(\boldsymbol{y}_t|\boldsymbol{x}_t)$ linear Gaussian, yielding closed-form importance weights. This ensures that standard SMC updates remain tractable. \cite{dou2024diffusion} further formalizes this approach within the Filtering for Posterior Sampling (FPS) framework. The resulting FPS-SMC algorithm provides a global consistency guaranty: as the number of particles approaches infinity, the approximation converges to the true Bayesian posterior. While SMC provides asymptotic guaranties, it suffers from the curse of dimensionality and requires
a sufficient number of samples.

\vspace{-0.75em} 
\paragraph{Diffusion Priors with MCMC Posterior Sampling.}
Another related line of work formulates posterior sampling with diffusion or score-based priors as Markov chain Monte Carlo (MCMC) targeting an explicit posterior distribution or a controlled approximation. Within this class, the PnP Unadjusted Langevin Algorithm (PnP-ULA)\citep{laumont2022bayesian} and PnP Monte Carlo (PMC) \citep{sun2024provable} follow a closely related principle: both construct training-free MCMC samplers by combining an explicit measurement likelihood with an implicit denoising or score-based prior. PnP-ULA uses Tweedie's formula to interpret a pretrained score network as an approximate prior score and applies unadjusted Langevin dynamics for Bayesian imaging. PMC and its annealed variant extend this idea more directly to score-based generative priors, providing non-asymptotic stationarity guaranties. A recent extension of PMC, zeroth-order annealed PnP Monte Carlo (ZO-APMC) \citep{sahin2026provable}, develops a derivative-free annealed PMC framework that only requires forward-model evaluations and a pretrained score-based prior, making MCMC posterior sampling applicable when likelihood gradients or pseudo-inverses are unavailable. Subsequent work studies the sensitivity of this sampler to mismatched measurement and prior models, providing a posterior discrepancy metric to quantify the deviation of the resulting sampling distribution from the target posterior \citep{renaud2024plugandplay}. Recent proximal MCMC theory further extends these guaranties to broader settings \citep{renaud2026from}. In a different direction, the PnP Split Gibbs Sampler (PnP-SGS) alternates between a data-consistency update and a denoising-prior update, providing a practical Gibbs-style framework for Bayesian imaging with implicit deep generative priors \citep{coeurdoux2024plug}.

\clearpage

\section{Preliminaries of Tweedie's Formula in Diffusion Models}
\label{appsec: Preliminaries}

\begin{proposition}[First-order Tweedie Approximation \citep{chung2023diffusion}]
\label{proposition: First-order Tweedie Approximation}
    For the case of VP-SDE sampling, $p(\boldsymbol{x}_0|\boldsymbol{x}_t)$ has the analytical posterior expectation, denoted as $\boldsymbol{x}_{0|t}(\boldsymbol{x}_t) \coloneqq \mathbb{E}[\boldsymbol{x}_0|\boldsymbol{x}_t]$, is given by \citep{chung2023diffusion}: 
    \begin{equation} \label{eq: first order tweedie} \boldsymbol{x}_{0}(\boldsymbol{x}_t) = \frac{1}{\sqrt{\alpha(t)}}\left(\boldsymbol{x}_t + (1-\alpha(t))\nabla{\boldsymbol{x}_t}\log p_t(\boldsymbol{x}_t)\right), 
    \end{equation}
    where $\alpha(t) = \exp(-\int_0^t \beta(s)\mathrm{d}s)$ is the accumulated noise variance. 
\end{proposition}
By replacing the $\nabla_{\boldsymbol{x}_t}\log{p(\boldsymbol{x}_t)}$ with the score network $\boldsymbol{s}_{\boldsymbol{\theta}}(\boldsymbol{x}_t,t)$, the posterior expectation can be approximated as $\hat{\boldsymbol{x}}_0\approx\frac{1}{\sqrt{\alpha(t)}}\big(\boldsymbol{x}_t+(1-\alpha(t))\boldsymbol{s}_{\boldsymbol{\theta}}(\boldsymbol{x}_t,t)\big)$.

\begin{proposition}[Second-order Tweedie Approximation \citep{meng2021estimating,boys2024tweedie}] 
\label{proposition: Second-order Tweedie Approximation} For the case of VP-SDE, the conditional posterior $p(\boldsymbol{x}_0 \mid \boldsymbol{x}_t)$ has an analytical posterior covariance given by:
\begin{equation*}
    \mathbb{V}(\boldsymbol{x}_0|\boldsymbol{x}_t)=(1-\alpha(t))\big(\frac{1}{\alpha(t)}\boldsymbol{I} +\nabla_{\boldsymbol{x}_t}^2\log{p(\boldsymbol{x}_t)}\big).
\end{equation*}
\end{proposition}
This follows from applying Tweedie’s formula to second-order moments. It introduces a curvature term to the first-order estimate, the Hessian of $\log p(\boldsymbol{x}_t)$—to capture posterior uncertainty. While the first-order formula gives the posterior mean $\mathbb{E}(\boldsymbol{x}_0|\boldsymbol{x}_t)$, the second-order formula gives the corresponding covariance $\mathbb{V}(\boldsymbol{x}_0|\boldsymbol{x}_t)$.  Direct evaluation of $\nabla_{\boldsymbol{x}_t}^2 \log p(\boldsymbol{x}_t)$ is often impractical because it requires second-order score information. In practice, this Hessian is approximated using the score network, either through automatic differentiation \citep{meng2021estimating,zhang2023moment,boys2024tweedie} or by training an auxiliary network to estimate the Hessian term \citep{bao2022estimating,bao2022analytic,ou2025improving}.

\begin{remark}
There is a straightforward algebraic relation between Proposition \ref{proposition: First-order Tweedie Approximation} and \ref{proposition: Second-order Tweedie Approximation} such that
    \begin{equation}
        \label{eq: relation between First-order Tweedie and Second-order Tweedie}
        \mathbb{V}(\boldsymbol{x}_0|\boldsymbol{x}_t)=\sqrt{\alpha(t)}J_{\boldsymbol{x}_0}(\boldsymbol{x}_t)+\left(\frac{1-2\alpha(t)}{\alpha(t)}\right)\boldsymbol{I},
    \end{equation}
    where \(J_{\boldsymbol{x}_0}\) denotes the Jacobian of the posterior mean map. Note that this identity is not valid for an arbitrary Jacobian: the right-hand side need not be symmetric positive semidefinite and thus may fail to define a covariance matrix. It holds when \(J_T\) is the Jacobian of the \emph{true} posterior-mean map induced by the associated density \(p(\boldsymbol{x}_t)\).
\end{remark}

\begin{proposition}[Conditional Tweedie's Formula]
    \label{prop: exact form of conditional score}
    Assuming a transition kernel $p(\boldsymbol{x}_t|\boldsymbol{x}_0) = \mathcal{N}(\boldsymbol{x}_t; \sqrt{\alpha(t)}\boldsymbol{x}_0, (1-\alpha(t))\boldsymbol{I})$ from VP-SDE, the score of the posterior marginal distribution at time $t$ is exactly related to the conditional posterior expectation of $\boldsymbol{x}_0$ by:
    \begin{equation}
    \nabla_{\boldsymbol{x}_t} \log p(\boldsymbol{x}_t|\boldsymbol{y}) = -\frac{\boldsymbol{x}_t - \sqrt{\alpha(t)}\mathbb{E}[\boldsymbol{x}_0|\boldsymbol{x}_t, \boldsymbol{y}]}{1 - \alpha(t)}.
    \end{equation}
\end{proposition}

\clearpage
\section{Proof of main results}
\label{appendix: Proof of main results}
\paragraph{Notation and setup.} 
\begin{itemize}
    \item \textbf{Diffusion process}. As stated in Section \ref{subsec: Diffusion Models for Inverse Problems}, we consider a VP-SDE with a noise schedule $\beta(t)$ satisfying $0 \le \beta_{\min} \le \beta(t) \le \beta_{\max} < \infty$. Defining $\alpha(t) \coloneqq \exp\left(-\int_0^t \beta(s)\mathrm{d}s\right)$, the transition kernel is $p(\boldsymbol{x}_t|\boldsymbol{x}_0) = \mathcal{N}(\boldsymbol{x}_t; \sqrt{\alpha(t)}\boldsymbol{x}_0, (1-\alpha(t))\boldsymbol{I})$. Note that $\alpha(t) \in [e^{-T\beta_{\max}}, 1]$.
    \item \textbf{Endpoint Truncation}. In practice, the reverse process is evaluated from \(T\) up to a small positive terminal time $\delta>0$.  We adopt the same truncated-time setting in our theoretical analysis so that all score-based and denoising quantities are considered on $t\in[\delta,T]$.
    \item \textbf{Norms}. \(\|\cdot\|\): Unless otherwise stated, $\|\boldsymbol{x}\|$ denotes the Euclidean norm for vectors, and $\|F\| \coloneqq \sup_{\|\boldsymbol{x}\|=1} \|F\boldsymbol{x}\|$ denotes the operator norm for matrices or operators.
    \item \textbf{Denoised estimates}. In the following, we let $\boldsymbol{x}_{0}(\boldsymbol{x}_t)$ denote the ideal denoised estimate computed using the true score function. We reserve $\hat{\boldsymbol{x}}_{0}(\boldsymbol{x}_t)$ for the approximate estimate derived from the learned score network $\boldsymbol{s}_{\boldsymbol{\theta}}$.

\end{itemize}

\subsection{Assumptions}
We introduce the following assumptions, commonly adopted in the theoretical analysis of DMs and BIPs, which will be used throughout the paper:

\begin{assumption}[Score Network Regularity]
\label{appendix assumption: Score Network Regularity}
The score network
$\boldsymbol{s}_{\boldsymbol{\theta}}(\cdot,t)$ is continuously differentiable
with respect to $\boldsymbol{x}$ for time $t\in[\delta,T]$, and there exists a constant $L_s>0$ such that
\(
\|\nabla_{\boldsymbol{x}}\boldsymbol{s}_{\boldsymbol{\theta}}
(\boldsymbol{x},t)\|
\le L_s.
\)
\end{assumption}

\begin{assumption}[Score Network Accuracy]
    \label{appendix assumption: Score Network Accuracy}
    The expected error of the score network is bounded for \(t\in[\delta,T]\): 
    \(\mathbb{E}\big[\|\nabla_{\boldsymbol{x}}\log{p_t(\boldsymbol{x})}-\boldsymbol{s}_{\boldsymbol{\theta}}(\boldsymbol{x},t)\|^2_2\big]\le\epsilon_s(t)^2
    \), where \(\int_{\delta}^T\epsilon_s(t)^2<\infty\).
\end{assumption}

\subsection{Proof of Lemma \ref{lemma: stability of DPS}}

\begin{lemma}
    \label{lemma: Likelihood Score Sensitivity}
    Suppose the likelihood $p(\boldsymbol{y}|\boldsymbol{x})$ is a Gaussian distribution such that $p(\boldsymbol{y}|\boldsymbol{x}) = \mathcal{N}(\boldsymbol{y}; F\boldsymbol{x}, \sigma_y^2\boldsymbol{I})$. The likelihood score $\nabla_{\boldsymbol{x}}\log p(\boldsymbol{y}|\boldsymbol{x})$ is uniformly Lipschitz continuous with respect to the measurement $\boldsymbol{y}$. Specifically, for any $\boldsymbol{x} \in \mathbb{R}^{d_x}$, we have: \[\|\nabla_{\boldsymbol{x}}\log{p(\boldsymbol{y}|\boldsymbol{x})-\nabla_{\boldsymbol{x}}\log{p(\boldsymbol{y}^\prime|\boldsymbol{x})}}\|\le L_{\text{lld}}\|\boldsymbol{y}-\boldsymbol{y}^\prime\|, \quad \forall \boldsymbol{y}, \boldsymbol{y}^\prime \in \mathbb{R}^{d_y},
    \]
    where $L_{\text{lld}} = \|F\|/\sigma_y^2$.
\end{lemma}
\begin{proof}
    The score of the Gaussian measurement model is given as $\nabla_{\boldsymbol{x}} \log p(\boldsymbol{y}|\boldsymbol{x}) = \nabla F(\boldsymbol{x})^\top \boldsymbol{\Sigma}^{-1} (\boldsymbol{y} - F(\boldsymbol{x})),$ where $\nabla F(\boldsymbol{x})$ is the Jacobian of the measurement model. We have
    \begin{align*}
        \|\nabla_{\boldsymbol{x}} \log p(\boldsymbol{y}|\boldsymbol{x}) - \nabla_{\boldsymbol{x}} \log p(\boldsymbol{y}^\prime|\boldsymbol{x})\|&=\frac{1}{\sigma_y^2}\|F^\top (\boldsymbol{y} - \bcancel{F\boldsymbol{x}}) - F^\top (\boldsymbol{y}^\prime - \bcancel{F\boldsymbol{x}})\|\\
        &=\frac{1}{\sigma_y^2}\| F^\top (\boldsymbol{y} - \boldsymbol{y}^\prime)\| \le\underbrace{\frac{1}{\sigma_y^2}\|F\|}_{L_{\text{lld}}} \cdot \|\boldsymbol{y} - \boldsymbol{y}^\prime\|.
    \end{align*}
\end{proof}

\begin{theorem}[Girsanov Theorem]
    \label{theorem: diffusion_model_KL_bound}
    For any pair of well-defined diffusion processes $\{\boldsymbol{x}_t\}_{t=\delta}^T$ and $\{\boldsymbol{x}^\prime_t\}_{t=\delta}^T$ on $\mathbb{R}^{d_x}$ defined as follows:
    \begin{align*}
        &\mathrm{d}\boldsymbol{x}=f(\boldsymbol{x}_t,t)\mathrm{d}t+g(t)\mathrm{d}\boldsymbol{w}_t \\
        &\mathrm{d}\boldsymbol{x}^\prime=f^\prime(\boldsymbol{x}^\prime_t,t)\mathrm{d}t+g(t)\mathrm{d}\boldsymbol{w}^\prime_t
    \end{align*}
    where $f,f^\prime\colon\mathbb{R}^{d_x}\times [\delta,T]\rightarrow\mathbb{R}^{d_x}$ are the two drift functions and $g\colon[\delta,T]\rightarrow\mathbb{R}_{>0}$is the diffusion function. Assume that: (1) \(g(t)>0\) for all \(t\in[\delta,T]\), (2) \(p_T\ll p_T'\), and (3) Novikov's condition holds \(
        \mathbb E\left[
        \exp\left(
        \frac12\int_\delta^T
        \left\|
        \frac{
        f(\boldsymbol{x}_t,t)-f'(\boldsymbol{x}_t,t)
        }{g(t)}
        \right\|^2
        \mathrm dt
        \right)
        \right]<\infty .
        \)
     Let $p_t$ and $p^\prime_t$ denote the distributions of $\boldsymbol{x}_t$ and $\boldsymbol{x}_t^\prime$ for any $t\in [\delta,T]$; then we have
    \begin{equation}
        \label{eq: diffusion KL bound}
         D_{KL}(p_0\|p^\prime_0)\le D_{KL}(p_T\|p^\prime_T)+\int_{0}^T\frac{1}{2g(t)^2}\mathbb{E}\big[\|f(\boldsymbol{x}_t,t)-f^\prime(\boldsymbol{x}_t,t)\|^2\big]\,\mathrm{d}t.
    \end{equation}
\end{theorem}
\textit{Proof.} The bound is a direct consequence of the Girsanov theorem, which relates the probability measures of diffusion processes under a change of drift. The result has been well established as standard in the analysis of generative diffusion models and is used to upper-bound the KL divergence between the true data distribution and the generated distribution by the time-accumulated mismatch between their respective score functions. For detailed derivations, we refer the reader to prior work, such as Theorem 1 in \citep{song2021maximum}, Lemma 2.22 in \citep{wu2024principled}, and Lemma C.1 in \citep{chen2025solving}. A version adapted for the probability flow ODE formulation can be found in Theorem 3.1 of \citep{lu2022maximum}.

In the VP-SDE, the reverse-time drift functions corresponding to the posterior $p_t(\boldsymbol{x}_t|\boldsymbol{y})$ given the measurement $\boldsymbol{y}$ are defined as:
\[
    f(\boldsymbol{x}_t,t) = -\beta(t)\left[\frac{\boldsymbol{x}_t}{2} + \nabla \log p_t(\boldsymbol{x_t}|\boldsymbol{y})\right].
\]
Substituting the diffusion coefficient $g(t) = \sqrt{\beta(t)}$, the squared difference appearing in the Girsanov Theorem (Equation \eqref{eq: diffusion KL bound}) simplifies to:
\[
    \frac{1}{2g(t)^2}\mathbb{E}\left[\|f(\boldsymbol{x}_t,t) - f^\prime(\boldsymbol{x}_t,t)\|^2\right] = \frac{\beta(t)}{2} \mathbb{E}\left[\|\nabla \log p_t(\boldsymbol{x}_t|\boldsymbol{y}) - \nabla \log p^\prime_t(\boldsymbol{x}_t|\boldsymbol{y})\|^2\right].
\]

\begin{lemma}
\label{lemma: Bounded Jacobian of the Posterior Mean}
Assume the forward process (VP) is defined by the SDE $\mathrm{d}\boldsymbol{x}_t = -\frac{1}{2}\beta(t)\boldsymbol{x}_t\,\mathrm{d}t+\sqrt{\beta(t)} \mathrm{d}\boldsymbol{w}_t$ as in Section \ref{subsec: Diffusion Models for Inverse Problems}. The Gaussian kernel is given by $\boldsymbol{x}_t|\boldsymbol{x}_0\sim\mathcal{N}(\sqrt{\alpha(t)}\boldsymbol{x}_0,\left(1-\alpha(t)\right)\boldsymbol{I})$, where $\alpha(t)=\exp{\left(-\int_0^t\beta(s)\,\mathrm{d}s\right)}$. Let $\hat{\boldsymbol{x}}_0(\boldsymbol{x}_t)$ be the Tweedie estimator approximating the posterior mean $\mathbb{E}[\boldsymbol{x}_0|\boldsymbol{x}_t]$ with the score network \(\boldsymbol{s}_{\boldsymbol{\theta}}\). Under Assumption \ref{appendix assumption: Score Network Regularity}, the Jacobian of this estimator is bounded:
\(\|J_{\hat{\boldsymbol{x}}_0}(\boldsymbol{x}_t)\|\le \exp{(T\beta_{\max}/2)}(1+L_s)\coloneq L_{\hat{\boldsymbol{x}}_0}\) for any time \(t\in[\delta,T]\). 
\end{lemma}
\begin{proof}
    The proof relies on Tweedie's formula, which relates the score of the marginal distribution  to the posterior expectation. For the Gaussian perturbation kernel defined above and Proposition \ref{proposition: First-order Tweedie Approximation}, the Tweedie estimator is given as \(\hat{\boldsymbol{x}}_0(\boldsymbol{x}_t) = \frac{1}{\sqrt{\alpha(t)}} \left(\boldsymbol{x}_t + \left((1-\alpha(t)\right) \boldsymbol{s}_{\boldsymbol{\theta}}(\boldsymbol{x}_t, t)\right). \) We compute the Jacobian $J_{\hat{\boldsymbol{x}}_0}(\boldsymbol{x}_t) = \nabla_{\boldsymbol{x}_t} \hat{\boldsymbol{x}}_0(\boldsymbol{x}_t)$ by differentiating with respect to $\boldsymbol{x}_t$; then, for any time \(t\in[\delta,T]\),
    \begin{align*}
    \|J_{\hat{\boldsymbol{x}}_0}(\boldsymbol{x}_t)\| & = \|\frac{1}{\sqrt{\alpha(t)}} \Big(\nabla_{\boldsymbol{x}_t} \boldsymbol{x}_t + \left(1-\alpha(t)\right) \nabla_{\boldsymbol{x}_t} \boldsymbol{s}_{\boldsymbol{\theta}}(\boldsymbol{x}_t, t)\Big)\|\\
    &=\frac{1}{\sqrt{\alpha(t)}}\|I+\left(1-\alpha(t)\right)\nabla_{\boldsymbol{x}_t} \boldsymbol{s}_{\boldsymbol{\theta}}(\boldsymbol{x}_t, t)\|\\
    &\le \frac{1}{\sqrt{\alpha(t)}}\Big(\|I\|+\left(1-\alpha(t)\right)\|\nabla_{\boldsymbol{x}_t} \boldsymbol{s}_{\boldsymbol{\theta}}(\boldsymbol{x}_t, t) \| \Big)\\
    &\le \frac{1}{\sqrt{\alpha(t)}}\big(1+\left(1-\alpha(t)\right)L_{s}\big)\\
    &\le\alpha(T)^{-1/2}(1+L_s)=\exp{(T\beta_{\max}/2)}(1+L_s)
\end{align*}
Moreover, \(\int_{\delta}^TL_{\hat{\boldsymbol{x}}_0}\mathrm{d}t\le(T-\delta)\exp{(T\beta_{\max}/2)}(1+L_s)<\infty\).
\end{proof}

\textbf{\textit{Proof} of Lemma \ref{lemma: stability of DPS}.} 
Given the measurement $\boldsymbol{y}\in\mathbb{R}^{d_y}$, the difference between the two drift terms in the reverse SDEs at time $t$ is
\begin{equation*}
    \begin{split}
        \nabla_{\boldsymbol{x}_t}\log{p_t(\boldsymbol{x}_t|\boldsymbol{y}_*)}- \nabla_{\boldsymbol{x}_t}\log{p_t(\boldsymbol{x}_t|\boldsymbol{y})}&=\bcancel{\boldsymbol{s}_{\boldsymbol{\theta}}(\boldsymbol{x}_t,t)}+\nabla_{\boldsymbol{x}_t}\log{\tilde{p}_t(\boldsymbol{y}_*|\boldsymbol{x}_t)}-\bcancel{\boldsymbol{s}_{\boldsymbol{\theta}}(\boldsymbol{x}_t,t)}-\nabla_{\boldsymbol{x}_t}\log{\tilde{p}_t(\boldsymbol{y}|\boldsymbol{x}_t)}\\
        &=\nabla_{\boldsymbol{x}_t}\log{\tilde{p}_t(\boldsymbol{y}_*|\boldsymbol{x}_t)}-\nabla_{\boldsymbol{x}_t}\log{\tilde{p}_t(\boldsymbol{y}|\boldsymbol{x}_t)}.
    \end{split}
\end{equation*}
By the chain rule for the DPS likelihood approximation,
\[
\nabla_{\boldsymbol{x}_t}\log \tilde{p}_t(\boldsymbol{y}\mid \boldsymbol{x}_t)
=
J_{\hat{\boldsymbol{x}}_0}(\boldsymbol{x}_t)^\top
\nabla_{\hat{\boldsymbol{x}}_{0|t}}
\log p(\boldsymbol{y}\mid \hat{\boldsymbol{x}}_{0|t}).
\]
The difference between the two drift terms in the reverse SDEs can be bounded as
\[
\|f(\boldsymbol{x}_t,t)-f'(\boldsymbol{x}_t,t)\|=\|-\beta(t)\left[\nabla_{\boldsymbol{x}_t}\log{\tilde{p}_t(\boldsymbol{y}_*|\boldsymbol{x}_t)}-\nabla_{\boldsymbol{x}_t}\log{\tilde{p}_t(\boldsymbol{y}|\boldsymbol{x}_t)}\right]\|
\le
\beta(t)L_{\hat{\boldsymbol{x}}_0}L_{\mathrm{lld}}
\|\boldsymbol{y}_*-\boldsymbol{y}\|. 
\]
Since \(g(t)=\sqrt{\beta(t)}\) for the VP-SDE and \(\beta(t)\le \beta_{\max}\), the above bound also provides square integrability:
\(
\|\boldsymbol{y}_*-\boldsymbol{y}\|^2\int_{\delta}^{T}
\frac{\beta(t)^2L_{\hat{\boldsymbol{x}}_0}^2}{g(t)^2}
\,\mathrm{d}t
=
\|\boldsymbol{y}_*-\boldsymbol{y}\|^2L_{\hat{\boldsymbol{x}}_0}^2L^2_{\text{lld}}\int_{\delta}^{T}
\beta(t)
\,\mathrm{d}t
<\infty
\) by lemmas~\ref{lemma: Likelihood Score Sensitivity} and~\ref{lemma: Bounded Jacobian of the Posterior Mean}. Hence, the Girsanov Theorem \ref{theorem: diffusion_model_KL_bound} can be applied
\[D_{KL}\big(p_\delta(\boldsymbol{x}|\boldsymbol{y}_*)\|p_\delta(\boldsymbol{x}|\boldsymbol{y})\big) \le D_T + \int_\delta^T\frac{\beta(t)^2}{2g(t)^2}\mathbb{E}\Big[\|\nabla_{\boldsymbol{x}_t}\log{p_t(\boldsymbol{x}_t|\boldsymbol{y}_*)}- \nabla_{\boldsymbol{x}_t}\log{p_t(\boldsymbol{x}_t|\boldsymbol{y})} \|^2 \Big] \mathrm{d}t.\] Since both initial distributions are standard Gaussian, the initial term \(D_T\) vanishes, yielding
\[
D_{\mathrm{KL}}\big(p_\delta(\cdot\mid\boldsymbol{y}_*)\,\|\,p_\delta(\cdot\mid\boldsymbol{y})\big)
\le C_{\mathrm{stable}}\|\boldsymbol{y}_*-\boldsymbol{y}\|^2,\,\,
C_{\mathrm{stable}}\coloneq\frac{1}{2}L_{\text{lld}}^2L_{\hat{\boldsymbol{x}}_0}^2\int_{\delta}^T\beta(t)\mathrm{d}t < \infty.
\]

\subsection{Proof of Theorem \ref{theorem: outlier-robust}}
\label{subsec appendix: Proof of Theorem  outlier-robust}
\begin{proof}[\textbf{Proof of Theorem \ref{theorem: outlier-robust}}]
The proof follows the same Girsanov argument as above. The only change is the
likelihood-score sensitivity bound. For the robust likelihood \(\tilde{p}^{(\ell)}\), the
approximate likelihood score is given by chain rule as 
\[
\nabla_{\boldsymbol{x}_t}\log \tilde p_t^{(\ell)}(\boldsymbol y\mid \boldsymbol x_t)
=
J_{\hat{\boldsymbol x}_0}(\boldsymbol x_t,t)^\top
\nabla_{\hat{\boldsymbol x}_{0|t}}
\ell_{\boldsymbol y}(\hat{\boldsymbol x}_{0|t}).
\]
where
\(
\ell_{\boldsymbol y}(\hat{\boldsymbol x}_{0|t})
=
\sum_{i=1}^{d_y}\rho(r_i)\), and \(
r_i
=
y_i-[ F\hat{\boldsymbol x}_{0|t}]_i.
\)
For the robust loss
\(
\rho(r)=\frac{1}{2\sigma_y^2}w(r)r^2,
\)
we have
\(
\psi(r):=\rho'(r)
=
\frac{1}{2\sigma_y^2}
\left(2r w(r)+r^2 w'(r)\right).
\)
Under the robust weight condition on \(w\) given in Equation \ref{eq: robust_weight_condition}, both \(|r w(r)|\) and \(|r^2w'(r)|\) are uniformly
bounded. Hence there exists \(K<\infty\) such that
\(
\sup_{r\in\mathbb R}|\psi(r)|\le K .
\)
Therefore,
\[
\left\|
\nabla_{\hat{\boldsymbol x}_{0|t}}
\ell_{\boldsymbol y}^{(\rho)}(\hat{\boldsymbol x}_{0|t})
\right\|
=
\left\|
\sum_{i=1}^{d_y}\psi(r_i)\nabla_{\hat{\boldsymbol x}_{0|t}} r_i
\right\|
\le
K \|F\| \sqrt{d_y},
\]
uniformly over \(\boldsymbol y\). Thus, for any \(\boldsymbol y\),
\(
\left\|
\nabla_{\hat{\boldsymbol x}_{0|t}}
\ell_{\boldsymbol y_*}
-
\nabla_{\hat{\boldsymbol x}_{0|t}}
\ell_{\boldsymbol y}
\right\|
\le
2K \|F\|\sqrt{d_y}.
\)
With Lemma \ref{lemma: Bounded Jacobian of the Posterior Mean}
we obtain
\[
\left\|
\nabla_{\boldsymbol{x}_t}\log \tilde p_t^{(\rho)}(\boldsymbol y_*\mid \boldsymbol{x}_t)
-
\nabla_{\boldsymbol{x}_t}\log \tilde p_t^{(\rho)}(\boldsymbol y\mid \boldsymbol{x}_t)
\right\|^2
\le
4L_{\hat{\boldsymbol x}_0}^2K^2L_F^2d_y .
\]
Applying Theorem~\ref{theorem: diffusion_model_KL_bound} and using
\(g(t)^2=\beta(t)\) for the VP-SDE gives
\[
\begin{aligned}
D_{\mathrm{KL}}
\left(
p_\delta^{(\ell)}(\cdot\mid \boldsymbol y_*)
\,\|\,
p_\delta^{(\ell)}(\cdot\mid \boldsymbol y)
\right)
&\le
\frac12
\int_\delta^T
\beta(t)
\,
4L_{\hat{\boldsymbol x}_0}^2K^2L_F^2d_y
\,\mathrm dt  \\
&\le
2(T-\delta)\beta_{\max}
L_{\hat{\boldsymbol x}_0}^2K^2L_F^2d_y<\infty .
\end{aligned}
\]
Therefore, the KL divergence under the robust likelihood is uniformly bounded
with respect to the measurement discrepancy, unlike the Gaussian likelihood
case whose bound grows with \(\|\boldsymbol y_*-\boldsymbol y\|^2\).

\end{proof}

\clearpage
\section{Supplementary about RDP}
\subsection{General Pseudo Algorithm}
We include the algorithm design of RDP for generic diffusion posterior sampling in this section.
\label{Appendix subsec: General Pseudo Algorithm}
\begin{algorithm}[H]
\caption{RDP Plug in for Generic Diffusion Posterior Sampling}
\label{alg: general robust diffusion posterior sampling}
\begin{algorithmic}[1]
\REQUIRE observation \(\boldsymbol{y}\), score network \(\boldsymbol{s}_{\boldsymbol{\theta}}\), noise schedule \(\{\beta_t\}_{t=0}^T\), weight function \(w(\cdot)\), temperature \(\{\tau_t\}_{t=0}^T\)
\STATE Initialize \(\boldsymbol{x}_T\sim\mathcal{N}(\boldsymbol{0},\boldsymbol{I})\)
\FOR{$t=T,...,1$}
    \STATE \(\alpha_t = 1 - \beta_t, \quad \bar{\alpha}_t = \prod_{i=1}^t \alpha_i\)
    \STATE \(\hat{\boldsymbol{x}}_{0|t} = \frac{1}{\sqrt{\bar{\alpha}_t}}(\boldsymbol{x}_t + (1-\bar{\alpha}_t)\boldsymbol{s}_{\boldsymbol{\theta}}(\boldsymbol{x}_t, t))\)
    \STATE \(\hat{\boldsymbol{y}} = F \hat{\boldsymbol{x}}_{0|t}\) and \(\boldsymbol{r} = \boldsymbol{y} - \hat{\boldsymbol{y}}\)
    \STATE \(\boldsymbol{W} = \text{diag}(w(r_1), \dots, w(r_m))\)
    \STATE \(\ell_{\boldsymbol{y}}(\boldsymbol{x}_t)=\boldsymbol{W}^\top\colorbox{green!20}{\(\log p(\boldsymbol{y}|\hat{\boldsymbol{x}}_0(\boldsymbol{x}_t))\)}\) 
    \STATE \(\boldsymbol{x}_{t-1}\)=\colorbox{green!20}{UpdateStep(\(\boldsymbol{x}_t,\boldsymbol{s}_{\boldsymbol{\theta}},\nabla\ell_{\boldsymbol{y}}(\boldsymbol{x}_t),\alpha_t,\beta_t\))}
\ENDFOR
\end{algorithmic}
\end{algorithm}

\begin{itemize}
    \item The likelihood term $p(\boldsymbol{y} \mid \hat{\boldsymbol{x}}_{0}(\boldsymbol{x}_t))$ enforces data consistency; however, its specific formulation depends on the noise and operator assumptions. Our method can be used to \textit{robustify} a broad family of algorithms. In this paper, we focus on the DPS algorithm, which approximates the likelihood via the surrogate mean $\hat{\boldsymbol{x}}_{0|t}$ as a point estimate. For $\Pi$GDM \citep{song2023pseudoinverse} (which we use in our experiments), the method utilizes SVD to define the likelihood in the spectral domain, modeled as $p(\boldsymbol{y}|\boldsymbol{x}_t)\approx\mathcal{N}(\boldsymbol{y} \mid F(\hat{\boldsymbol{x}}_{0|t}), r_t\boldsymbol{I})$. We integrate RDP by reweighting the log-likelihood objective: $\ell_{\text{RDP}} \approx \boldsymbol{W}^\top \log \mathcal{N}(\boldsymbol{y} \mid F(\hat{\boldsymbol{x}}_{0|t}), r_t\boldsymbol{I})$, while additionally incorporating the annealing variance \(r_t\). Furthermore, \cite{boys2024tweedie} proposes a method for matching the variance using the second-order Tweedie's formula. This higher-order approximation of variance can similarly be incorporated with RDP.

    \item The \textsc{UpdateStep} is solver-agnostic, allowing robust guidance to be integrated into various reverse-time discretizations. This includes standard SDE and ODE probability flow sampling, Predictor-Corrector (PC) schemes where Langevin dynamics refine the posterior estimate via $\nabla \ell_{\boldsymbol{y}}$, as well as recent progress in recent midpoint estimation approaches such as \cite{janati2024divide,moufad2025variational}. Recent midpoint estimation such as \cite{janati2024divide,moufad2025variational}, and Furthermore, the guidance term seamlessly extends to high-order solvers—such as DPM-Solver++ \citep{lu2025dpm} or Heun's method \citep{karras2022elucidating}, enabling efficient sample generation with fewer steps.
    
\end{itemize}

\subsection{Other Weight functions}
\label{Appendix subsec: Other Weight function choices Discussion}
\begin{itemize}
    \item \textbf{Mahalanobis-based distance.} When the measurement noise exhibits a nontrivial heterogeneous structure (e.g., $\sigma_i\neq\sigma_j$ for $i\neq j$ in the Gaussian case), it is natural to measure residuals using Mahalanobis-based (MB) distance with a diagonal covariance $\Sigma=(\sigma_1,...,\sigma_{d_y})$, leading to weights \(w_i(y_i,\hat{y}_{t,i})=(1+\frac{(y_i-\hat{y}_{t,i})^2}{\sigma_ic^2})^{-\frac{1}{2}}\). 
    \item \textbf{Global scale.} Many diffusion samplers choose the step size based on the residual norm. For example, \cite{chung2023diffusion} suggests choosing 
    this step size to be inversely proportional to the measurement error norm 
    $\|\boldsymbol{y}-F(\hat{\boldsymbol{x}}_0)\|_2$ \citep{daras2024survey}. Within the GB framework, this corresponds to taking a global constant \(w=\frac{c}{\varepsilon + \|\boldsymbol{y}-\hat{\boldsymbol{y}}\|_2}\) for some constant $c\in(0,1]$ and a small $\varepsilon>0$. This stabilizes the reverse process by preventing excessively large gradient steps. However, because it uses global scaling, it down-weights clean observations along with outliers and can thus be viewed as a special case of RDP.
    \item \textbf{Huber (\(L_1\) Loss).} Another well known robust loss function is Huberized Loss \citep{huber2011robust}, which combines the sensitivity of squared error ($L_2$) for small residuals with the robustness of absolute error ($L_1$) for large residuals. For a residual $r_i = y_i - \hat{y}_{t,i}$ and a specified transition threshold $c$, i.e.,
    \[    \ell_c(r_i) =
    \begin{cases}
    \frac{1}{2}r_i^2, & |r_i| \le c, \\[3pt]
    c\left(|r_i|-\frac{1}{2}c\right), & \text{otherwise} .
    \end{cases}\]
    This loss preserves the quadratic Gaussian penalty for small residuals, while clipping the gradient magnitude to \(c\) for large residuals. As a result, large corrupted measurements are prevented from dominating the guidance.

\end{itemize}

\clearpage
\section{Additional Experiment Details}
\label{sec: Additional Experiment Details}

This appendix provides additional details on the experimental setup, including task definitions, measurement models, noise configurations, hyperparameters, and implementation details. All experiments were run on two NVIDIA \texttt{A100} GPUs with \texttt{40GiB} memory.

\subsection{Tasks}
We evaluate our proposed method across a diverse set of inverse problems, ranging from standard linear image restoration tasks to complex nonlinear scientific imaging. The measurement process is generally modeled as $\boldsymbol{y} = F(\boldsymbol{x}) + \boldsymbol{n}$, where $F(\cdot)$ represents the measurement model and $\boldsymbol{\epsilon}$ denotes the measurement noise.

\begin{itemize}
    \item (Linear) Inverse Scattering: This task represents a linearized version of the general nonlinear inverse scattering problem. The measurement model is given by:
    \(
    \boldsymbol{y} = \mathbf{H}(\boldsymbol{u}_{\text{in}} \odot \boldsymbol{x}) + \boldsymbol{\epsilon},
    \)
    where $\boldsymbol{u}_{\text{in}}$ is the known incident light field, $\mathbf{H}$ is the Green's function propagator, and $\odot$ denotes the element-wise product. See more details in Section 3 of \cite{zheng2025inversebench}

    \item Image Inpainting: The measurement model $F$ is a diagonal masking matrix $\mathbf{M} \in \{0, 1\}^d$, which effectively removes a random fraction of the pixels. The measurement is given by $\boldsymbol{y} = \mathbf{M} \odot \boldsymbol{x} + \boldsymbol{\epsilon}$. The mask \(\mathbf{M}\) consists of a $32\times32$ box placed at a random location within the image.

    \item (Gaussian) Deblurring: The measurement model $F$ corresponds to a convolution with an isotropic Gaussian kernel. The kernel is chosen with a standard deviation $\sigma_{\text{blur}}=3$ and a kernel size $61 \times 61$. 

    \item Phase Retrieval: A critical problem in coherent diffractive imaging (e.g., X-ray crystallography), where detectors capture only the intensity of the diffracted wave. The measurement model is defined as $\mathbf{y} = |\mathcal{T}(\mathbf{x})| + \mathbf{n}$, where $\mathcal{T}$ denotes the 2D discrete Fourier transform and $|\cdot|$ is the element-wise magnitude operator. This problem is severely ill-posed due to the loss of phase information.
\end{itemize}
\begin{figure}[hb]
    \centering
    \includegraphics[width=0.8\linewidth]{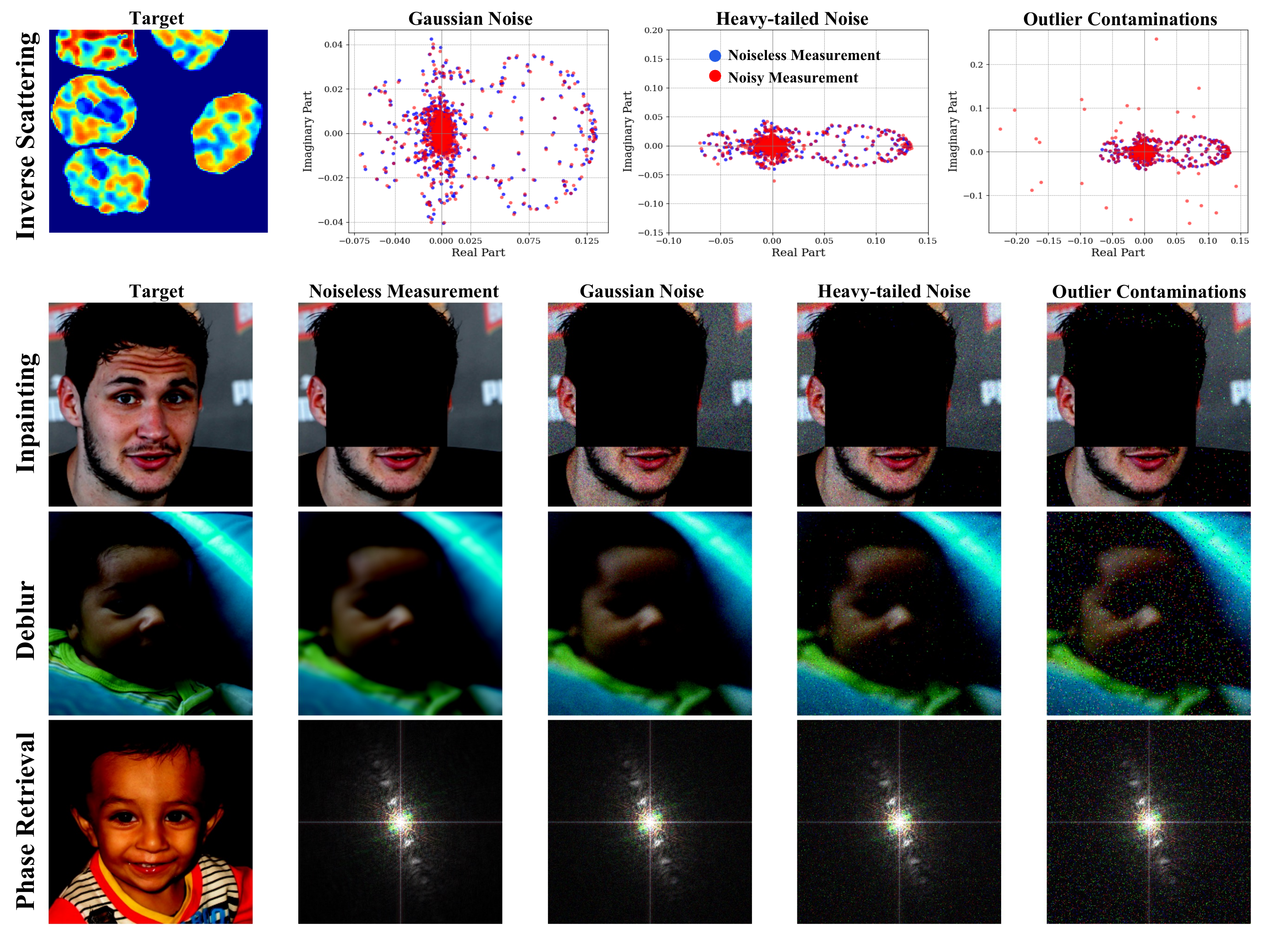}
    \caption{Visualization of measurements across three noise schemes (left$\rightarrow$right: Gaussian, heavy-tailed, outlier).}
    \label{fig: all_obs}
\end{figure}
\clearpage
\subsection{Experimental Design}
\label{subsec: Noise models}
In our numerical experiments in Section~\ref{sec: experiments}, we assess the robustness of the proposed modular approaches under three representative settings that commonly arise in scientific inverse problems:
\begin{itemize}
    \item Well-specified Likelihood (Gaussian Noise): 
    We utilize standard additive white Gaussian noise as the baseline control. This represents the ideal setting where the likelihood assumed by standard solvers matches the true noise distribution.
    
    \item Distributional Misspecification (Heavy-tailed Noise): 
    To test robustness against violations of the Gaussian assumption, we model noise using a Student-$t$ distribution. The scale parameter is calibrated to match the standard deviation of the corresponding Gaussian baseline, while the degrees of freedom $\nu$ control the tail heaviness. Heavy-tailed errors are extensively studied in robust estimation problems such as in \cite{aravkin2014robust} and \cite{janjic2018representation}. 

    \item Outlier Contamination (Sparse Impulsive Noise): 
    We model corruption by replacing a fraction of the measurements with extreme, arbitrary values. 
    The severity is controlled by two parameters: the corruption ratio $p$ (the percentage of contaminated measurements) and the amplification factor $m$ (the magnitude of the outlier values). This setting challenges the solver to identify and reject localized corruptions, simulating practical issues such as heat noise spikes or data transmission scattering errors. To simulate outlier corruptions, we first sample random noise from the base Gaussian distribution and then scale the magnitude of $p\%$ of the total dimensions by a factor of $m$.
\end{itemize}

\begin{table}[h]
    \centering
    \caption{Noise Schemes Summary. Well-specified lists the standard deviation $\sigma_y$ used for Gaussian baselines. Misspecification details the Student-$t$ parameters ($\nu$ and calibrated scale). Contamination specifies the outlier percentage ($p$) and the magnitude factor ($m$), where $m$ denotes the multiplier of the signal's peak dynamic range.}
    \resizebox{\linewidth}{!}{
    \begin{tabular}{c|c|c|c}
        \toprule
        & Well-Specified ($\mathcal{N}(0, \sigma_y^2)$)& Misspecification (Student-$t$) & Contamination (Outliers $p$, Mag $m$)\\
        \midrule
        IS & $\sigma_y = 10^{-3}$ & $\nu=2.2$ & $p=1\%$, $m=30\times$ \\
        FFHQ-I & $\sigma_y = 0.05$ & $\nu=2.5$ & $p=5\%$, $m=30\times$ \\
        FFHQ-D & $\sigma_y = 0.05$ & $\nu=2.5$ & $p=5\%$, $m=30\times$ \\
        FFHQ-PR & $\sigma_y = 10^{-3}$ & $\nu=2.5$ & $p=1\%$, $m=10\times$ \\
        \bottomrule
    \end{tabular}
    }
    \label{tab:noise-schemes}
\end{table}





\subsection{Pretrained Diffusion Model Details}
We use pretrained diffusion models with checkpoints provided by \textit{InverseBench} \citep{zheng2025inversebench}, which employ U-Net backbones, as in \cite{song2021score}. Detailed configurations are listed in Table \ref{tab: network configurations}.
\begin{table}[H]
    \centering
     \caption{Model network configurations for pre-trained diffusion models.}
    \begin{tabular}{cccccc}
    \midrule
         & Inverse Scattering  & FFHQ  \\
    \noalign{\hrule}
    Input resolution   &  $128\times128$  &    $3\times256\times256$      \\
    \# Attention blocks in encoder/decoder  & $5$ &  $6$\\
    \# Residual blocks per resolution & $1$ &  $3$ \\
    Attention resolutions & $16$ & $16$ \\
    \# Parameters & 26.8M & 93.6M \\
    \midrule
    \end{tabular}
    \label{tab: network configurations}
\end{table}

\begin{table}[H]
    \centering
    \caption{Hyperparameter tuning for different sampling methods and tasks. For other hyperparameters in PnPDM, we found it does not have significant effects to the results and we kept the configurations with \cite{zheng2025inversebench}}
      \begin{adjustbox}{width=\textwidth} 
    \begin{tabular}{lccccccccc}
        \toprule
        & \multicolumn{2}{c}{Inverse Scattering} & \multicolumn{2}{c}{FFHQ-Inpaint} & \multicolumn{2}{c}{FFHQ-Deblur} & \multicolumn{2}{c}{FFHQ-PR} \\
        \cmidrule(lr){2-3} \cmidrule(lr){4-5} \cmidrule(lr){6-7} \cmidrule(lr){8-9}
        \textbf{Method} & Used & Range & Used & Range & Used & Range & Used & Range \\
        \midrule
        DPS &  &  &  &  & & & \\
        Guidance Scale & $350$ & $[100,500]$ & $1$ & $[0.1,50]$ & $5$ & $[0.1,50]$  & $0.5$ & $[0.1,50]$ \\
        \noalign{\hrule}
        LGD &  &  &  &  & & & \\
        Guidance Scale &  $2500$ & $[500, 5000]$ & $1$ & $[0.1,50]$ & $5$ & $[0.1,100]$  & $1$ & $[0.1,100]$ \\
        \noalign{\hrule}
        $\Pi$GDM &  &  &  &  & & &  \\
        Stochasticity $\eta$ &  $0.3$ & $(0, 1]$ & $0.95$ & $(0, 1]$ & $0.85$ & $(0, 1]$ & - & - \\
        \noalign{\hrule}
        PnPDM &  &  &  &  &  & &\\
        Annealing Max $\rho_{\text{max}}$&  $20$ & $[1, 100]$  &  $20$ & $[1, 100]$  &  $10$ & $[1, 100]$ & $10$ & $[1, 100]$\\
        Annealing decay rate $\rho$  & $0.9$ & $[0.5,1)$  & $0.9$ & $[0.5,1)$ & $0.9$ & $[0.5,1)$ & $0.9$ & $[0.5,1)$ \\
        Langevin step size $\tau$ & $5\times10^{-4}$ & $(0,1]$ &   $5\times10^{-4}$ & $(0,1]$  & $5\times10^{-4}$ & $(0,1]$  & $1\times10^{-4}$ & $(0,1]$\\
        \noalign{\hrule}
            DiffPIR &  &  &  &  &  &  &  &  \\
        Regularization $\lambda$ 
        & $4\times10^{-4}$ & $[10^{-5},10^{-2}]$ 
        & $1$ & $[10^{-1},10^{4}]$ 
        & $8$ & $[10^{-1},10^{4}]$  
        & $8$ & $[10^{-1},10^{4}]$ \\
        Stochasticity $\eta$ 
        & $1$ & $[0,1]$ 
        & $1$ & $[0,1]$ 
        & $0.5$ & $[0,1]$  
        & $0.5$ & $[0,1]$ \\
        \bottomrule
    \end{tabular}
      \end{adjustbox}
    \label{tab: hyperparameter_tuning}
\end{table}

\begin{figure}
    \centering
    \includegraphics[width=1\linewidth]{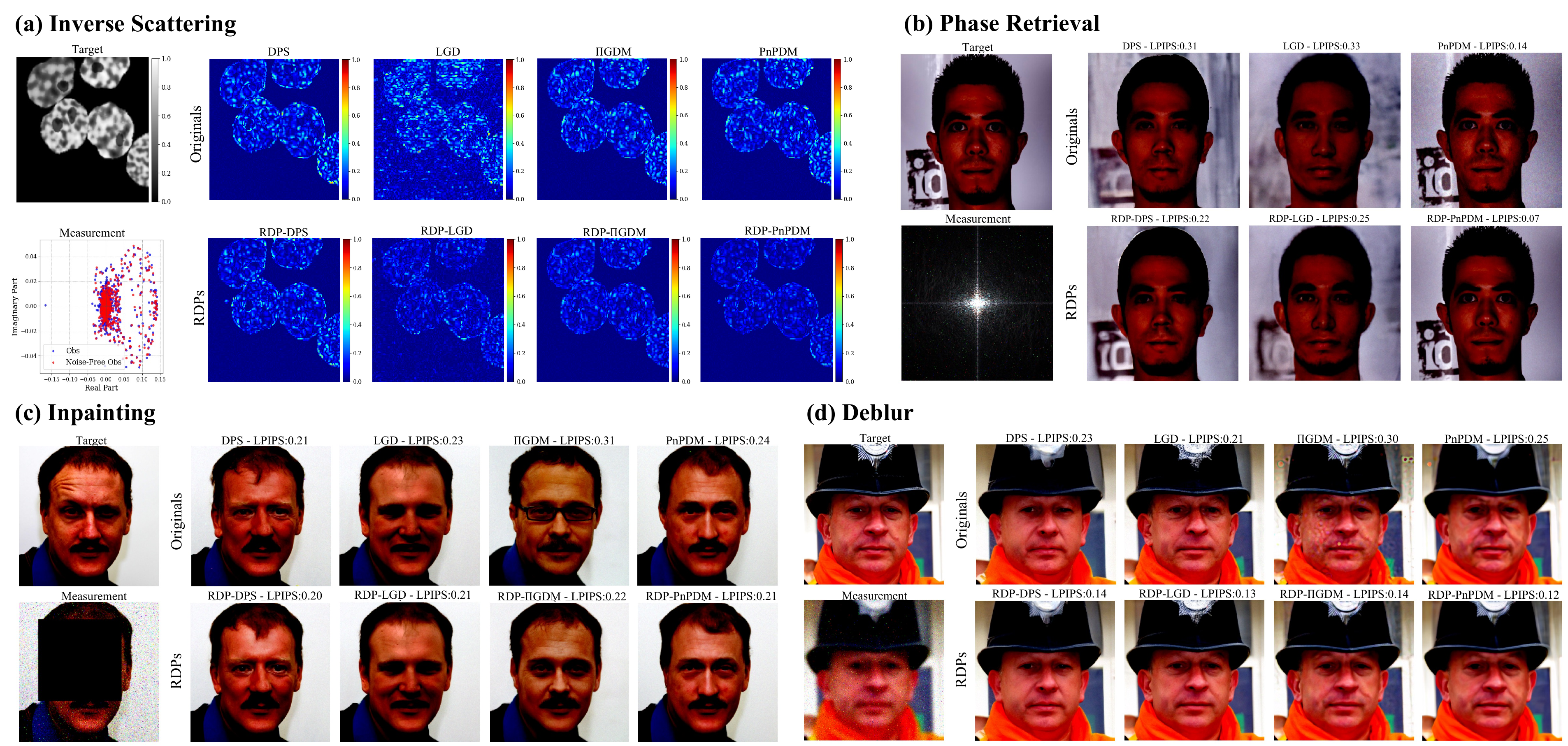}
    \caption{Qualitative results on all noise misspecification tasks: (a) Inverse Scattering, (b) Phase Retrieval, (c) Deblurring, and (d) Inpainting. In each panel, the leftmost column shows the ground truth and the contaminated observation. The remaining columns compare the original methods (top row) with their robustified versions using our sampler (bottom row); we report reconstruction absolute errors in (a) and reconstructed samples in (b)–(d).}
    \label{fig: combined_qual_figure_student_t}
\end{figure}

\begin{figure}[t]
  \centering
  \includegraphics[width=0.74\textwidth]{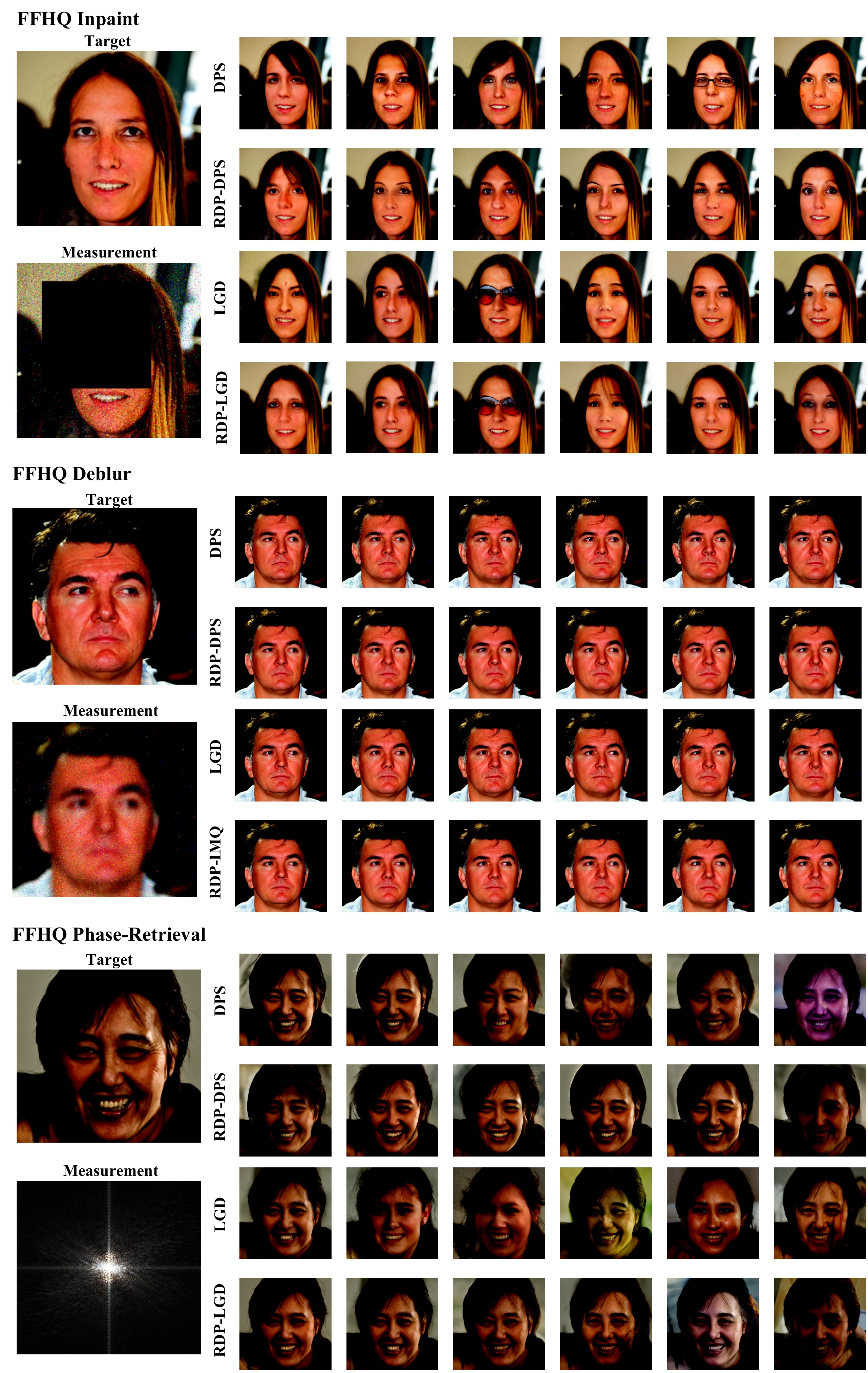}
  \caption{Multiple posterior samples from DPS and LGP with and without the RDPs on FFHQ, given heavy-tailed noises.}
\end{figure}

\begin{figure}[t]
  \centering
  \includegraphics[width=0.74\textwidth]{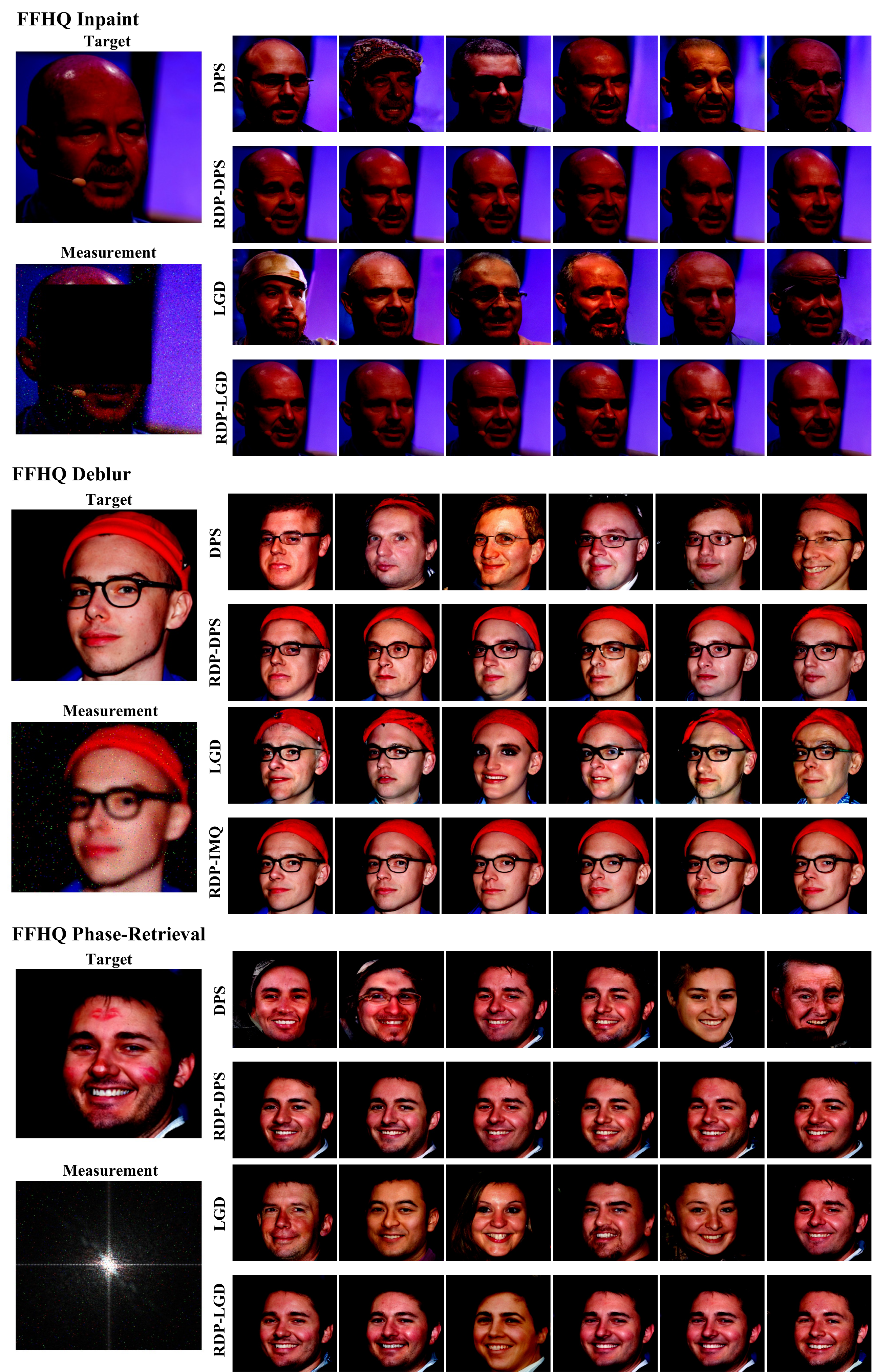}
  \caption{Multiple posterior samples from DPS and LGP with and without the RDPs on FFHQ, given observations corrupted by outliers.}
\end{figure}

\clearpage
\subsection{Additional Experimental Results}
\label{appendix: additional experimental results}
\paragraph{Performance under varying measurement dimension.} In practice, the measurement dimension \(d_y\), which determines how much information is available from observations, is often more influential for inverse problems than the target dimension \(d_x\). We conduct an ablation on inverse scattering (IS) task, where the number of receivers recording the scattered wavefield controls the amount of observed data and hence changes \(d_y\). Specifically, we evaluate \(d_y \in \{360,180,90\}\).

Under Gaussian noise, the gap between DPS and RDP-DPS remains small, consistent with the main results. In the other two settings, however, the advantage of RDP becomes more pronounced as \(d_y\) decreases. This suggests that robust guidance is more important when measurements are limited, since each corrupted measurement has a larger influence on the posterior update.

\begin{table}[htbp]
  \centering
  \vspace{-1em}
  \caption{Performances under different noise settings and receiver configurations}
  \label{tab:Performances under different noise settings and receiver configurations}
      \begin{adjustbox}{width=\textwidth} 
  \begin{tabular}{lcccccccccc}
    \toprule
    \multirow{2}{*}{Method} & \multirow{2}{*}{\makecell{Num of Recivers/\\$d_y$}} & \multicolumn{3}{c}{Well-specification (Gaussian)} & \multicolumn{3}{c}{Noise Misspecification (Student-T)} & \multicolumn{3}{c}{Contamination of Outliers} \\
    \cmidrule(lr){3-5} \cmidrule(lr){6-8} \cmidrule(lr){9-11}
    & & PSNR & SSIM & NMAE & PSNR & SSIM & NMAE & PSNR & SSIM & NMAE \\
    \midrule
    DPS     & 360/7200 & 26.26$\pm$2.83 & 0.86$\pm$0.02 & 0.16$\pm$0.02 & 21.53$\pm$2.93 & 0.78$\pm$0.06 & 0.25$\pm$0.05 & 19.37$\pm$3.44 & 0.71$\pm$0.10 & 0.33$\pm$0.05 \\
    RDP-DPS & 360/7200 & 26.01$\pm$2.72 & 0.86$\pm$0.02 & 0.15$\pm$0.01 & 23.84$\pm$2.59 & 0.82$\pm$0.03 & 0.19$\pm$0.04 & 25.24$\pm$2.74 & 0.83$\pm$0.02 & 0.15$\pm$0.02 \\
    \midrule
    DPS     & 180/3600 & 24.50$\pm$2.90 & 0.84$\pm$0.03 & 0.19$\pm$0.02 & 19.68$\pm$3.15 & 0.73$\pm$0.07 & 0.31$\pm$0.06 & 17.28$\pm$3.68 & 0.63$\pm$0.12 & 0.41$\pm$0.06 \\
    RDP-DPS & 180/3600 & 24.20$\pm$2.78 & 0.83$\pm$0.03 & 0.18$\pm$0.01 & 22.50$\pm$2.65 & 0.79$\pm$0.04 & 0.22$\pm$0.04 & 23.58$\pm$2.82 & 0.80$\pm$0.03 & 0.17$\pm$0.02 \\
    \midrule
    DPS     & 90/1800  & 22.76$\pm$3.05 & 0.81$\pm$0.04 & 0.23$\pm$0.03 & 17.56$\pm$3.42 & 0.68$\pm$0.08 & 0.37$\pm$0.07 & 15.76$\pm$3.95 & 0.56$\pm$0.13 & 0.47$\pm$0.07 \\
    RDP-DPS & 90/1800  & 22.40$\pm$2.86 & 0.80$\pm$0.04 & 0.22$\pm$0.02 & 20.98$\pm$2.78 & 0.75$\pm$0.05 & 0.26$\pm$0.05 & 22.54$\pm$2.95 & 0.77$\pm$0.04 & 0.20$\pm$0.03 \\
    \bottomrule
  \end{tabular}
  \end{adjustbox}
  \vspace{-0.5em}
\end{table}

\begin{figure}[ht]
    \centering
    \includegraphics[width=0.5\linewidth]{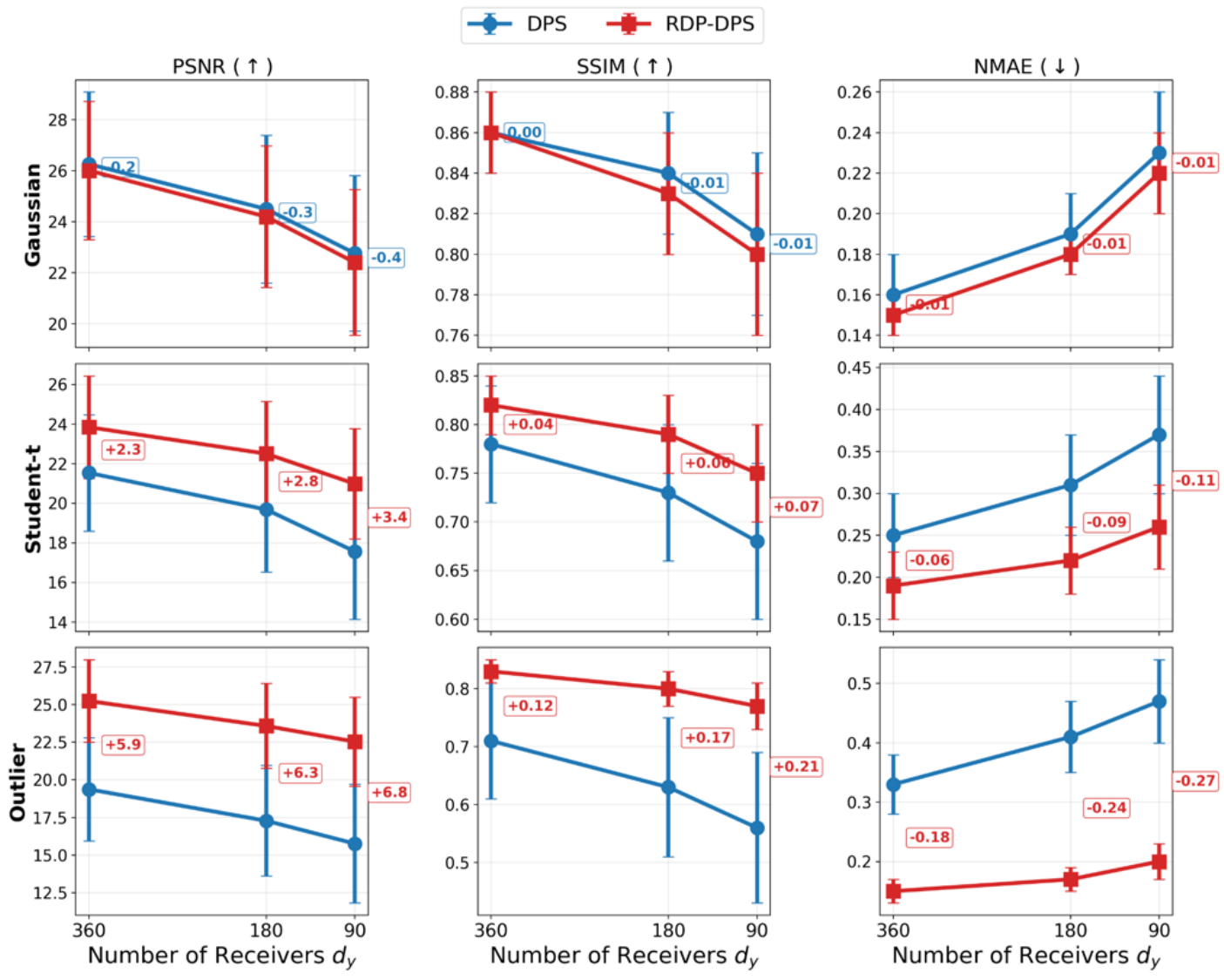}
    \caption{Recovery performance of DPS and RDP-DPS as the measurement dimension \(d_y\) varies under three noise settings on the IS task. RDP-DPS shows larger gains when measurements are limited and corrupted.}
    \label{fig:placeholder}
    \vspace{-1em}
\end{figure}

\vspace{-0.5em}
\paragraph{Performances under noise mismatch levels.} To evaluate a more practical form of model misspecification, we further consider moderate noise-level mismatch on the IS task. Specifically, the sampler assumes a Gaussian likelihood with noise standard deviation \(\sigma\), while the true measurement noise has standard deviation \(2\sigma\), \(4\sigma\), or \(8\sigma\). Results are reported in \Cref{tab:Performances under mismatch levels}.

The advantage of RDP-DPS increases as the mismatch level becomes larger. For example, RDP-DPS achieves comparable performance under \(2\times\) mismatch and provides clearer improvements under \(4\times\) and \(8\times\) mismatch, especially in NMAE. The gains are smaller than those observed under outlier contamination or Student-\(t\) noise, which is expected because noise-level mismatch affects all measurements uniformly rather than producing clearly distinguishable corrupted components. In this case, component-wise reweighting is less directly targeted, but it still mitigates the effect of an overconfident likelihood and provides consistent improvements.

\begin{table}[htbp]
  \centering
  \small
  \caption{Performances under mismatch levels.}
  \label{tab:Performances under mismatch levels}
  \begin{tabular}{lcccc}
    \toprule
    Method & $\sigma_y$ mismatch & PSNR & SSIM & NMAE \\
    \midrule
    DPS     & 2$\times$ & 23.64$\pm$2.27 & 0.81$\pm$0.03 & 0.23$\pm$0.01 \\
    RDP-DPS & 2$\times$ & 23.72$\pm$2.30 & 0.82$\pm$0.02 & 0.23$\pm$0.01 \\
    DPS     & 4$\times$ & 20.58$\pm$2.19 & 0.75$\pm$0.05 & 0.33$\pm$0.01 \\
    RDP-DPS & 4$\times$ & 20.88$\pm$2.16 & 0.76$\pm$0.04 & 0.30$\pm$0.02 \\
    DPS     & 8$\times$ & 17.94$\pm$2.30 & 0.69$\pm$0.07 & 0.40$\pm$0.03 \\
    RDP-DPS & 8$\times$ & 18.18$\pm$2.22 & 0.71$\pm$0.07 & 0.37$\pm$0.03 \\
    \bottomrule
  \end{tabular}
  \vspace{-1em}
\end{table}

\vspace{-0.5em}
\paragraph{Wall-clock time.}
We further evaluate the computational overhead introduced by RDP on the IS task. Since RDP only adds a component-wise residual weighting to the guidance term, it incurs negligible additional cost per sampling step. As shown in \Cref{tab:Wall-clock time per sample}, RDP-DPS has nearly identical wall-clock time to DPS.

\begin{table}[htbp]
  \centering
  \small
  \caption{Wall-clock time per sample on the IS task, averaged over 100 samples.}
  \label{tab:Wall-clock time per sample}
  \begin{tabular}{lccc}
    \toprule
    Method & Samples & Mean & Std \\
    \midrule
    DPS     & 100 & 30.2s & 1.4s \\
    RDP-DPS & 100 & 30.4s & 1.3s \\
    \bottomrule
  \end{tabular}
  \vspace{-1em}
\end{table}


\paragraph{Ablation study on weighting functions and quantiles. } The choice of robust weighting function can affect the recovery quality. To examine whether the performance of RDP depends critically on the specific IMQ weighting, we conduct ablations on the IS task over two design choices: the weighting function and the adaptive quantile \(q\).

In addition to IMQ, we evaluate a Huber-type weighting function:
\[
w(r) =
\begin{cases}
1, & |r| \leq c, \\[4pt]
c / |r|, & |r| > c.
\end{cases}
\]
where the transition threshold \(c\) is set adaptively as the 75th percentile of the residual magnitudes \(|r|\), following the same calibration strategy as IMQ.

The results show that IMQ and Huber weighting perform comparably under Gaussian and Student-\(t\) noise, while IMQ gives a slight advantage under outlier contamination. Since both weights use quantile-calibrated thresholds, they suppress a similar set of large-residual measurements, although IMQ provides a smoother transition.
\begin{table}[htbp]
  \centering
  \small
    \vspace{-1em}
  \caption{Ablation on weighting function on IS.}
  \label{tab:r2_weighting_ablation}
  \begin{tabular}{llccc}
    \toprule
    Weighting & Noise    & PSNR          & SSIM          & NMAE          \\
    \midrule
    IMQ       & Gaussian & 26.01$\pm$2.72 & 0.86$\pm$0.02 & 0.15$\pm$0.01 \\
    Huber     & Gaussian & 26.00$\pm$2.45 & 0.86$\pm$0.01 & 0.15$\pm$0.01 \\
    IMQ       & Student-T & 23.84$\pm$2.59 & 0.82$\pm$0.03 & 0.19$\pm$0.04 \\
    Huber     & Student-T & 23.82$\pm$2.41 & 0.81$\pm$0.02 & 0.20$\pm$0.01 \\
    IMQ       & Outliers & 25.24$\pm$2.74 & 0.83$\pm$0.02 & 0.15$\pm$0.02 \\
    Huber     & Outliers & 24.73$\pm$2.51 & 0.83$\pm$0.02 & 0.17$\pm$0.01 \\
    \bottomrule
  \end{tabular}
  \vspace{-1em}
\end{table}

We further evaluate the IMQ quantile \(q\in\{0.10,0.25,0.50,0.75,0.95\}\). Performance is stable for \(q\geq0.5\), whereas \(q=0.10\) is overly aggressive and degrades performance under outlier contamination by down-weighting useful clean measurements. These results indicate that RDP is not highly sensitive to the weighting design and support IMQ with \(q=0.75\) as a reasonable choice.
\begin{table}[htbp]
  \centering
  \small
  \caption{Ablation on quantile $q$ of IMQ on IS.}
  \label{tab:r3_quantile_ablation}
  \begin{adjustbox}{max width=\textwidth}
  \begin{tabular}{lcccccc}
    \toprule
    \multirow{2}{*}{$q$} & \multicolumn{3}{c}{Student-T} & \multicolumn{3}{c}{Outliers} \\
    \cmidrule(lr){2-4} \cmidrule(lr){5-7}
    & PSNR & SSIM & NMAE & PSNR & SSIM & NMAE \\
    \midrule
    0.10 & 23.47$\pm$2.64 & 0.81$\pm$0.04 & 0.21$\pm$0.04 & 22.63$\pm$2.81 & 0.80$\pm$0.03 & 0.19$\pm$0.02 \\
    0.25 & 23.59$\pm$2.60 & 0.82$\pm$0.03 & 0.20$\pm$0.03 & 24.96$\pm$2.76 & 0.82$\pm$0.02 & 0.16$\pm$0.02 \\
    0.50 & 23.81$\pm$2.58 & 0.82$\pm$0.03 & 0.19$\pm$0.04 & 25.24$\pm$2.76 & 0.83$\pm$0.02 & 0.15$\pm$0.02 \\
    0.75 & 23.85$\pm$2.59 & 0.82$\pm$0.03 & 0.19$\pm$0.04 & 25.24$\pm$2.74 & 0.83$\pm$0.02 & 0.15$\pm$0.02 \\
    0.95 & 23.83$\pm$2.60 & 0.82$\pm$0.04 & 0.19$\pm$0.04 & 25.25$\pm$2.73 & 0.83$\pm$0.02 & 0.15$\pm$0.02 \\
    \bottomrule
  \end{tabular}
  \end{adjustbox}
    \vspace{-2em}
\end{table}



\end{document}